\documentclass[sigconf, screen]{acmart}

\AtBeginDocument{%
  }

\usepackage{pifont}% http://ctan.org/pkg/pifont
\newcommand{\cmark}{\ding{51}}%
\newcommand{\xmark}{\ding{55}}%

\usepackage{hyperref}
\usepackage{url}
\usepackage{bm}
\usepackage{enumitem}
\usepackage{wrapfig}
\usepackage{booktabs}
\usepackage{multirow}
\usepackage{caption}
\usepackage{graphicx}
\usepackage{subfigure}
\usepackage{wrapfig}
\usepackage{mathtools}
\usepackage{amsmath,amsfonts, amsthm}
\usepackage{tikz}
\usepackage{algorithm}
\usepackage{algpseudocode}
\usepackage{threeparttable}
\usepackage{xspace}
\usepackage{tcolorbox}
\usepackage{subcaption}

\usepackage{stackengine,scalerel,amssymb}
\def\dclesize{\ThisStyle{\raisebox{-.7pt}{\scalebox{1.45}{$\SavedStyle\bigcirc$}}}}
\def\dcle{\ensurestackMath{\stackon[0pt]{\geqq}{\dclesize}}}
\def\cle{\def\stacktype{L}\mathbin{\scalerel*{\dcle}{\dclesize}}}

\DeclareMathOperator*{\argmin}{min}

\definecolor{burgundy}{rgb}{0.7, 0.0, 0.13}
\definecolor{forestgreen}{rgb}{0.13, 0.55, 0.13}
\definecolor{flame}{rgb}{0.89, 0.35, 0.13}
\usepackage{array}
\usepackage{arydshln}
\algnewcommand{\LineComment}[1]{\State{\textcolor{blue}{\(\triangleright\) #1}}}

\setcopyright{acmlicensed}
\copyrightyear{2018}
\acmYear{2018}
\acmDOI{XXXXXXX.XXXXXXX}
\acmConference[Conference acronym 'XX]{Make sure to enter the correct
  conference title from your rights confirmation emai}{June 03--05,
  2018}{Woodstock, NY}
\acmISBN{978-1-4503-XXXX-X/18/06}

\begin{document}

% \title{DeepNT: Path-Centric Graph Neural Network for Network Tomography}
\title{Network Tomography with Path-Centric Graph Neural Network}

\author{Yuntong Hu}
\email{yuntong.hu@emory.edu}
\orcid{0000-0003-3802-9039}
\affiliation{%
  \institution{Emory University}
  \city{Atlanta}
  \state{GA}
  \country{USA}
}

\author{Junxinag Wang}
\email{junxiang.wang@alumni.emory.edu}
\orcid{0000-0002-6635-4296}
\affiliation{%
  \institution{NEC Labs America}
  \city{Princeton}
  \state{NJ}
  \country{USA}
}

\author{Liang Zhao}
\email{liang.zhao@emory.edu}
\orcid{0002-2648-9989}
\affiliation{%
  \institution{Emory University}
  \city{Atlanta}
  \state{GA}
  \country{USA}
}

\renewcommand{\shortauthors}{Yuntong et al.}

\begin{abstract}
Network tomography is a crucial problem in network monitoring, where the observable path performance metric values are used to infer the unobserved ones, making it essential for tasks such as route selection, fault diagnosis, and traffic control.
However, most existing methods either assume complete knowledge of network topology and metric formulas—an unrealistic expectation in many real-world scenarios with limited observability—or rely entirely on black-box end-to-end models.
To tackle this, in this paper, we argue that a good network tomography requires synergizing the knowledge from both data and appropriate inductive bias from (partial) prior knowledge. To see this, we propose Deep Network Tomography (DeepNT), a novel framework that leverages a path-centric graph neural network to predict path performance metrics without relying on predefined hand-crafted metrics, assumptions, or the real network topology. The path-centric graph neural network learns the path embedding by inferring and aggregating the embeddings of the sequence of nodes that compose this path. Training path-centric graph neural networks requires learning the neural netowrk parameters and network topology under discrete constraints induced by the observed path performance metrics, which motivates us to design a learning objective that imposes connectivity and sparsity constraints on topology and path performance triangle inequality on path performance. Extensive experiments on real-world and synthetic datasets demonstrate the superiority of DeepNT in predicting performance metrics and inferring graph topology compared to state-of-the-art methods.
\end{abstract}




\received{20 February 2007}
\received[revised]{12 March 2009}
\received[accepted]{5 June 2009}

\maketitle

\section{Introduction}

Network tomography seeks to infer unobserved network characteristics using those that are observed. More specifically, one may observe path performance metrics (PPMs), such as path delay and capacity, by measuring the two endpoints of the path. Hence, network tomography can use the observations of the PPMs of some pairs of endpoints to infer those of the remaining pairs, because many PPMs can be written as aggregations of corresponding measures on the edges which are typically far fewer the paths they can make up. Network tomography plays a crucial role in applications such as route selection \citep{ikeuchi2022network, tao2024delay}, fault diagnosis \citep{ramanan2015nettomo, qiao2020robust}, and traffic control \citep{zhang2018network, pan2020edge, lev2023traffic}.
In real-world applications, many network characteristics are inaccessible. For example, in a local area network connected to the Internet, internal devices remain hidden due to security policies (\hyperref[problem]{Figure}~\ref{problem}), necessitating network tomography to infer path-related states like latency and congestion \citep{cao2012routing}. Similar challenges arise in transportation networks for estimating arrival times and traffic conditions \citep{zhang2018network}, and in social networks for uncovering hidden connections \citep{xing2009state}.

\begin{figure}[t!]
  \centering
  \includegraphics[width=0.45\textwidth]{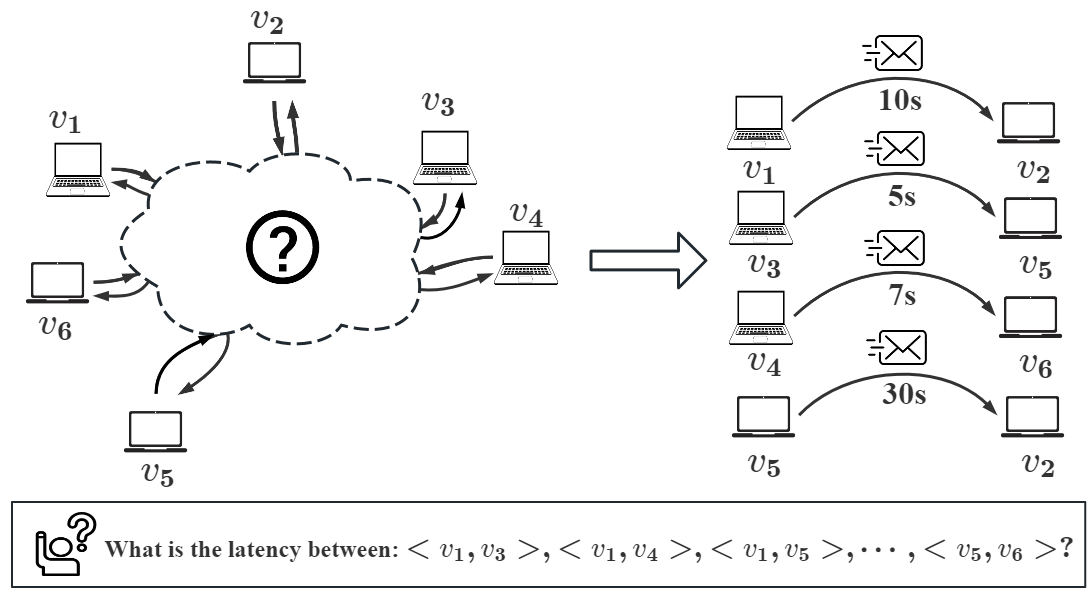}
  \caption{An illustration of network tomography in a sample network, where the end-to-end latency needs to be predicted when the network topology is not available.}
  \label{problem}
\vspace{-15pt}
\end{figure}

Network tomography is very challenging since the PPM values of a pair of nodes are jointly determined by the specific path in a particular network topology under a certain PPM type. To make this problem solvable, traditional network tomography approaches rely on the observed network topology and predefined, hand-crafted PPM calculations, focusing on either additive metrics, where the combined metric over a path is the sum of the involved link metrics (e.g., delay), or non-additive metrics, where the path performance is a nonlinear combination of link metrics \citep{feng2020bound, xue2022paint}. Other prescribed methods depend on assumptions like rare simultaneous failures \citep{lai2000measuring,carter1996measuring,jain2002end}, minimal sets of network failures \citep{duffield2003simple,duffield2006network,kompella2007detection}, or sparse performance metrics \citep{zhang2009spatio,firooz2010network,xu2011compressive}. These methods rely heavily on human-defined heuristics and rules, making their inference limited and biased by human domain knowledge, especially for many areas where we do not know what PPMs best model the network process.
For instance, a heuristic rule that assumes rare simultaneous network failures may be effective for localizing network bottlenecks in a computer network, but would not be suitable for environments like cloud computing or distributed systems, where performance degradation often involves multiple simultaneous disruptions across different nodes or links. 
More recently, dynamic routing \citep{tagyo2021network, sartzetakis2023network} and deep learning approaches \citep{ma2020neural, sartzetakis2022machine, tao2023network} have attempted to bypass the need for prior knowledge on PPMs by directly learning end-to-end models from data (e.g., predicting PPMs given the path's two endpoints). Hence, although they avoided traditional methods' heavy dependency on the observed network topology and prior knowledge of PPMs, they went to the other extreme, by typically completely overlooking the prior knowledge of the PPMs and the intrinsic relation between paths and edges.
%More importantly, both traditional methods and current deep learning-based approaches for network tomography are unable to simultaneously infer the network topology and predict path performance metrics. This limitation reduces their applicability in real-world scenarios, where network topology is often incomplete or unavailable.

To overcome the complementary drawbacks of traditional and deep learning-based methods, we pursue our method, Deep Network Tomography (\textbf{DeepNT}), which can infer the network topology and how it determines the PPMs, by deeply characterizing the network process by eliciting and synergizing the knowledge from both training data and partial knowledge of the inductive bias of PPMs. More concretely, we propose a new path-centric graph neural network that can infer the PPMs values of a path by learning its path embedding composed by the inferred sequence of node embeddings along this path. Training path-centric graph neural networks requires learning the neural network parameters and network topology following mixed connectivity constraint and path performance traignle inequailty constraints.
% \textit{How can path performance be predicted and topology be inferred when the network is incomplete?} We initialize our model using the observed portion of the network topology and use graph neural networks to iteratively aggregate neighborhood information into node embeddings. A path aggregation layer adaptively captures useful details in sampled paths. The adjacency matrix is updated by enforcing connectivity and sparsity constraints.
% \textit{How can the model incorporate prior path knowledge for different performance metrics?} We introduce metric-general bounds for unmeasured paths, accommodating both additive and non-additive performance metrics. By leveraging this prior path knowledge, our model adapts to various metrics, ensuring that the estimated values remain within appropriate bounds based on the sampled paths, thus enhancing both robustness and accuracy across different network contexts.

In summary, our primary contributions are as follows:
\begin{itemize}[leftmargin=*]
\item \textbf{New Problem.} We formulate the learning-based network tomography problem as learning representations for end-node pairs to simplify the optimization and identify unique challenges that arise in its real-world applications.
\item \textbf{New Computational Framework.} We propose a novel model for inferring unavailable adjacency matrices and metrics of unmeasured paths, learning end-node pair representations in an end-to-end manner.
\item \textbf{New Optimization Algorithm.} We propose to infer adjacency matrices with strongly and weakly connectivity constraints and a triangle inequality constraint. We proposed to transfer the discrete connectivity constraints into continuous forms for numerical optimization.
\item \textbf{Extensive Experiment Evaluation.} Extensive experiments on real-world and synthetic datasets demonstrate the outstanding performance of DeepNT. DeepNT outperforms other state-of-the-art models in predicting different path performance metrics as well as reconstructing network adjacency matrices.
\end{itemize}

\section{Related Work}\label{sec:related work}

\subsection{Network Tomography}
Network Tomography involves inferring internal network characteristics using performance metrics, which can be broadly classified as additive or non-additive. \textit{Additive metrics} frame the network tomography problem as a linear inverse problem, often assuming a known network topology and link-path relationships \citep{gurewitz2001estimating,liang2003maximum,chen2010network}. Statistical methods such as Maximum Likelihood Estimation (MLE) \citep{wandong2011research,teng2024learning}, Expectation Maximization (EM) \citep{bu2002network,wei2007mobile,wandong2011research}, and Bayesian estimation \citep{zhang2006origin,wandong2011research} are employed to solve this problem. Algebraic approaches, such as System of Linear Equations (SLE) \citep{bejerano2003robust,chen2003tomography,gopalan2011identifying} and Singular Value Decomposition (SVD) \citep{chua2005efficient,song2008netquest}, that rely on traceroute work well in certain scenarios but are often blocked by network providers to maintain the confidentiality of their routing strategies. When link performance metrics are sparse, compressive sensing techniques are used to identify all sparse link metrics \citep{firooz2010network,xu2011compressive}. Furthermore, studies have explored the sufficient and necessary conditions to identify all link performance metrics with minimal measurements \citep{gopalan2011identifying,alon2014economical}. \textit{Non-additive metrics}, such as boolean metrics, introduce additional complexity and constraints. These studies often assume that multiple simultaneous failures are rare, focusing on identifying network bottlenecks \citep{bejerano2003robust,horton2003number}. However, the assumption of rare simultaneous failures is not always valid. Some works address this by identifying the minimum set of network failures or reducing the number of measurements required \citep{duffield2006network,zeng2012automatic,ikeuchi2022network}. Additionally, several papers have proposed conditions and algorithms to efficiently detect network failures \citep{he2018distributed,nicola2018tight,bartolini2020fundamental,ibraheem2023network}, and some studies have attempted to apply deep learning to this field \citep{ma2020neural, sartzetakis2022machine, tao2023network}. However, most existing works rely on hand-crafted rules and specific assumptions, making them specialized for certain applications and unsuitable where prior knowledge of network properties or topology is unavailable.

\subsection{GNNs for Graph Structure Learning}
GNNs for Graph Structure Learning can be classified into approaches for learning discrete graph structures (i.e., binary adjacency matrices) and weighted graph structures (i.e., weighted adjacency matrices). Discrete graph structure approaches typically sample discrete structures from learned probabilistic adjacency matrices and subsequently feed these graphs into GNN models. Notable methods in this category include variational inference \citep{chen2018structured}, bilevel optimization \citep{franceschi2019learning}, and reinforcement learning \citep{kazemi2020representation}. However, the non-differentiability of discrete graph structures poses significant challenges, leading to the adoption of weighted graph structures, which encode richer edge information. A common approach involves establishing graph similarity metrics based on the assumption that node embeddings during training will resemble those during inference. Popular similarity metrics include cosine similarity \citep{nguyen2010cosine}, radial basis function (RBF) kernel \citep{yeung2007kernel}, and attention mechanisms \citep{chorowski2015attention}. While graph similarity techniques are applied in fully-connected graphs, graph sparsification techniques explicitly enforce sparsity to better reflect the characteristics of real-world graphs \citep{chen2020reinforcement, jin2020graph}. Additionally, graph regularization is employed in GNN models to enhance generalization and robustness \citep{chen2020iterative}. In this work, we leverage GNNs to learn end-node pair representations, enabling simultaneous prediction of path performance metrics and inference of the network topology.

\begin{figure*}[tb]
  \centering
  % \captionsetup{font=scriptsize} % Font size for the main captions
  \includegraphics[width=\textwidth]{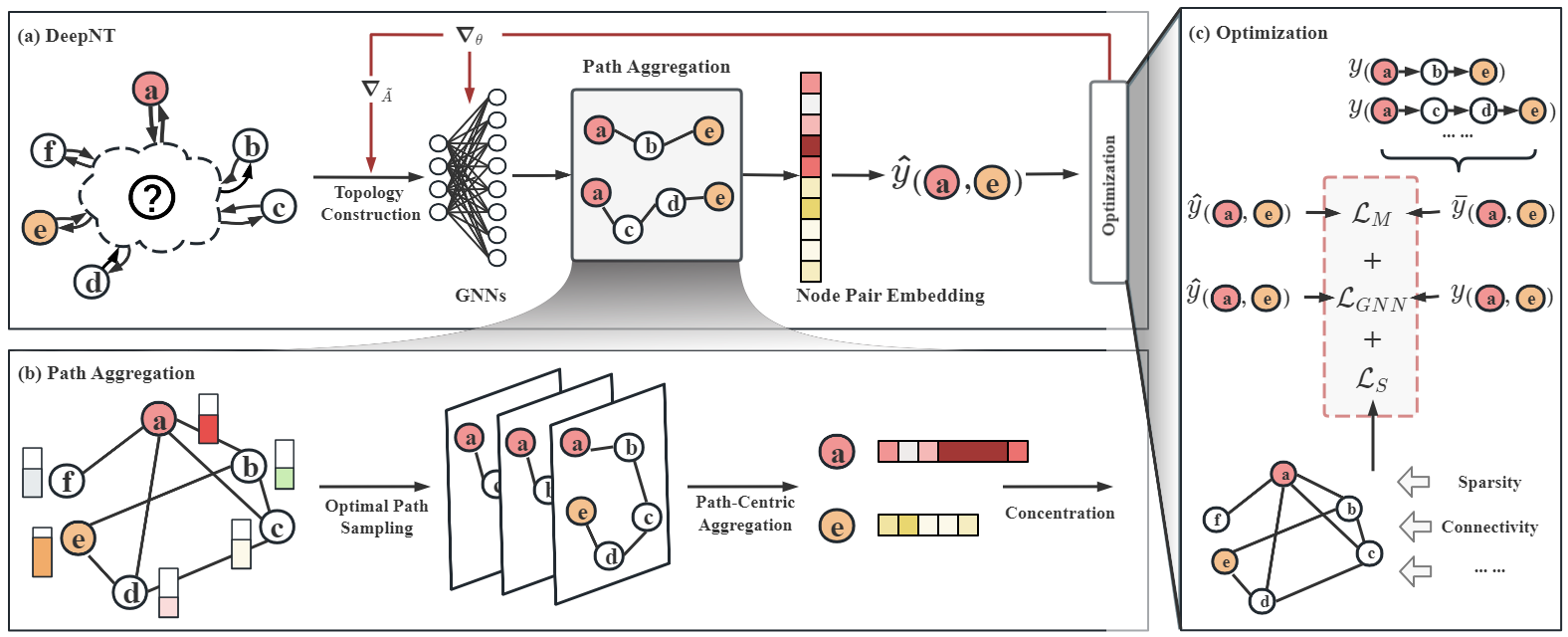}
  \caption{Overall framework of proposed deep network tomography solution.}
  \label{Flow}
  \vspace{-7pt}
\end{figure*}

\section{Problem Formulation}\label{sec:problem}

\textbf{Graphs}. A connected network $\mathcal G$ is defined as \(\mathcal G= (V, E) \), where \( V \) and $E \subseteq V \times V$ represent the node set and edge set, respectively, let \( A \) denote the adjacency matrix of \( \mathcal{G} \).

\vspace{2pt}
\noindent\textbf{Path Performance Metrics (PPMs).} Given a graph $\mathcal{G}= (V, E)$, let $\mathcal{P}_{uv} = \{p_{uv}^n\}_{n=1}^{N}$ represent the set of all possible paths from node $u$ to $v$, where $u, v \in V$, and $N$ denotes the number of possible paths between $u$ and $v$.
% A path in graph $\mathcal{G}$ from a source node $v$ to a target node $u$ is denoted by $p_{uv} = [v_1, v_2, \dots, v_k]$, where $v_1 =v$, $v_k =u$, and $(v_i, v_{i+1}) \in E$ for $i \in\{ 1,\dots,k-1\}$. The length of one path is $|p_{uv}| = k-1$, defined by the number of edges it contains. 
Let $y_{e}$ be the performance metric value of an individual edge, where $e\in E$. The path performance metric value is defined as the cumulative performance of all edges on a path, where the cumulative calculation depends on the type of metric being considered. The unified path performance metric is defined as follows:
\begin{equation}
    y_{uv}^n = \mathop{\bigotimes}\limits_{e_i\in p_{uv}^n} y_{e_i},~\text{where}~\bigotimes \in \{\sum, \prod, \bigwedge, \bigvee, \min, \cdots\},
\end{equation}
where $\bigotimes$ represents an operator that varies based on the type of path performance metrics, such as additive, multiplicative, boolean, min/max, etc. The optimal path performance between two nodes is defined as $y_{uv} = \{y_{uv}^n \mid y_{uv}^n \cle y_{uv}^k, ~\forall p_{uv}^k \neq p_{uv}^n \in \mathcal{P}_{uv}\}$, where $\cle$ indicates better performance, depending on the specific type of path performance metric.
For additive metrics like latency, \(\bigotimes = \sum\) and \(\cle = \leq\), where the path metric is the sum of edge latencies, with lower values indicating better performance.
For min metrics like capacity, \(\bigotimes = \min\) and \(\cle = \max\), where the minimum capacity along the path defines overall path performance, and higher values are preferable.

\vspace{2pt}
\noindent\textbf{Network Tomography.}  Define \(T = \{\langle u, v \rangle\}_{u \neq v \in V} \) as the set of all node pairs, where \( |T| = \binom{|V|}{2} \). Let \( S \subset T \) be a subset of node pairs for which the end-to-end optimal PPMs are measured. The exact path information between any two nodes is unknown.
% This means that paths are observed, with the source and target nodes corresponding to the node pair in $S$. 
Network tomography aims to use the measured PPMs values in \( S \) to predict the end-to-end optimal PPMs value of unmeasured node pairs in \( T \setminus S \). The optimal path between two nodes is typically determined by the Best Performance Routing (BPR). For instance, in a computer network, with measured end-to-end transmission delays for certain node pairs $S \subset T$, the goal of network tomography is to infer the minimum delays for the unmeasured pairs in $T \setminus S$, when the exact path information for node pairs in $T \setminus S$ is unknown.

\section{Deep Network Tomography}

% \textbf{Overview.} As shown in Figure {\color{burgundy}\ref{Flow}(a)}, only the end-to-end performance of certain node pairs is measured for an additive metric. For instance, we can measure \(y(\oc,\pc) = 6\), which actually corresponds to \(y(\oc \rightarrow \gc \rightarrow \pc) = y(\oc\rightarrow\gc) + y(\gc\rightarrow\pc) = 4 + 2\), but the path \(\oc \rightarrow \gc \rightarrow \pc\) remains unknown since the complete topology is unavailable. We utilize GNNs coupled with a path aggregation layer (as shown in Figure {\color{burgundy}\ref{Flow}(b)}) to learn path-centric end-node pair representations with incomplete adjacency matrices. To address the challenge of unknown network topology, we introduce a unified objective function framework to train the GNNs (as shown in Figure {\color{burgundy}\ref{Flow}(c)}) and infer the actual network topology. DeepNT reconstructs the network topology in an iterative manner, preserving the sparsity and connectivity properties of the network. The network topology and GNN parameters are updated alternately. After each topology update, the GNN parameters are updated on the reconstructed network.

\textbf{Overview.} Network tomography is very challenging since the PPM values of a pair of nodes are jointly determined by the specific path in a particular network topology under a certain PPM type. Hence, we propose a DeepNT to jointly infer network topology, consider path candidates, and learn path performance metrics, in order to effectively predict the PPM values of a node pair. Specifically, we propose a path-centric graph neural network to learn the candidate paths' embedding from the embeddings of the nodes on them and then aggregate them into node pair embedding for the final PPM value prediction, as illustrated in \hyperref[Flow]{Figure} \hyperref[Flow]{2(a)} and detailed in \hyperref[sec:gnn]{Section} \ref{sec:gnn}. To infer the network topology, DeepNT introduces a learning objective (\hyperref[Flow]{Figure} \hyperref[Flow]{2(c)}) that updates the adjacency matrix of the network by imposing constraints on connectivity and sparsity, as detailed in \hyperref[sec:topology]{Section} \ref{sec:topology}. This allows for the simultaneous prediction of PPM values and inference of network structure. Moreover, to leverage the inductive bias inherent in different types of PPMs, we introduce path performance triangle inequalities that further refine our predictions, as outlined in \hyperref[sec:bound]{Section} \ref{sec:bound}.

\subsection{Path-Centric Graph Neural Network}\label{sec:gnn}
% Given a graph $\mathcal{G}(V, E)$ and the estimated adjacent matrix $\tilde{A}$. A GNN is used to aggregate the information and update node embeddings.
% $\mathbf{h}_v^{(0)}$ is the initialization representation of node $v$, the final node embedding at layer $K$ is $\mathbf{h}_v^{(K)}$ and simplified as $\mathbf{h}_v$. 
% While node embeddings capture local neighborhood information, they may not fully encode the specific context of a path that connects two nodes. 
\textbf{Candidate paths' information elicitation and encoding.} Due to unknown paths from partial network topology, we aggregate multiple potential paths to capture PPM context. Specifically, for each node pair, e.g., \(\langle u, v \rangle\), we leverage BPR to sample $N$ loopless paths between them based on the adjacent matrix $\tilde{A}$, denoted as \(\mathcal{P}_{uv}^L = \{p_{uv}^{(n)}\}_{n \in [1, N]}\), ensuring that the lengths do not exceed $L$. Then, the node embeddings of $u$ and $v$ are updated with a path aggregation layer. For node \( v \), we first compute the attention scores between \( v \) and all nodes along the path from \( v \) to \( u \):
\begin{equation}
    e^{(n)}_{vz}=\sigma(r^\top[(h_v, h^{(n)}_z]) \Longrightarrow~ \alpha^{(n)}_{vz} = \text{softmax}(e^{(n)}_{vz}),
\end{equation}
where \( z \in p^{(n)}_{uv} \), \( \sigma \) is an activation function, and \( r \) is a predefined vector. We then aggregate path information to obtain path-centric node pair embeddings:
\begin{align}
   &\hat{h}^{(n)}_v=h_v+\sigma\bigg(\sum\nolimits_{z \in p^{(n)}_{uv}}\alpha^{(n)}_{vz} h^{(n)}_z\bigg), \\
   &\hat{h}_v=\text{READOUT}(\{\hat{h}^{(n)}_v \mid p^{(n)}_{uv} \in \mathcal{P}_{uv}^L \}),
\end{align}
where \( \text{READOUT}(\cdot) \) is a permutation-invariant function. $\hat{h}_u$ is obtained in the same way. This path-centric embedding aggregates the local neighborhood information of the end-node pair as well as the information in potential optimal paths connecting them. Finally, the concatenated representation of the end-node pair is passed through a projection module to predict the performance metric value $\hat{y}_{uv}$ via $f_{\theta}: (u, v, \tilde{A}) \rightarrow \hat{y}_{uv}$. The objective function can be formulated as,
\begin{equation}\label{eq:gnn}
    \min_{\theta, \tilde{A}} \mathcal{L}_{GNN}(\theta, \tilde{A}, Y) = \sum_{u,v\in V} l(f_{\theta}(u, v, \tilde{A}), y_{uv}),~s.t.,~\tilde{A} \in \mathcal{A},
\vspace{-3pt}
\end{equation}
where $\theta$ indicates the parameters of $f(\theta)$,  $l(\cdot,\cdot)$ is to measure the difference between the prediction $f_{\theta}(u, v, \tilde{A})$ and the target value $y_{uv}$, e.g., cross entropy for boolean metrics and $\ell_2$ norm for additive metrics. Another objective will be introduced in following \hyperref[sec:topology]{Section} \ref{sec:topology} to infer a optimal symmetric adjacency matrix \(\tilde{A} \in \mathcal{A}\) where $\mathcal{A}$ represents the set of valid adjacency matrices specified in \hyperref[sec:topology]{Section} \ref{sec:topology}. 
We then introduce a training penalty that constrains the DeepNT model with a path performance triangle inequality, applicable to any type of path performance metric.

Let $\mathcal{AP}$ denote the set of all possible paths between node pair combinations. DeepNT-$\mathcal{AP}$, utilizing the proposed path-centric aggregation with access to $\mathcal{AP}$, effectively distinguishes node pairs.

\begin{theorem}[DeepNT-$\mathcal{AP}$ Distinguishes Node Pairs Beyond 1-WL]\label{th3}
Let $G = (V, E)$, and let $\langle u, v \rangle$ and $\langle u', v' \rangle$ be two node pairs in $V \times V$ such that the local neighborhoods of $u$ and $u'$ are identical up to $L$ hops, and similarly for $v$ and $v'$. If the sets of paths $\mathcal{P}_{uv}^L$ and $\mathcal{P}_{u'v'}^L$ are different, DeepNT will assign distinct embeddings $f_{\text{DeepNT}}(u, v) \neq f_{\text{DeepNT}}(u', v')$.
\end{theorem}

The analysis of DeepNT's expressive power and the proof of \hyperref[th3]{Theorem} \ref{th3} are provided in \hyperref[expressive]{Appendix} \ref{expressive}.

\subsection{Path Performance Triangle Inequality} \label{sec:bound}
Since PPM is for the optimal path among all the paths between the node pairs, the performance of any path between these two nodes cannot exceed the performance of the observed optimal path. For each pair of nodes \(u\) and \(v\), and for any other node \(z\) on the path, the following triangle inequality must hold: \(y_{uv} \cle (y_{uz} \bigotimes y_{zv}\)) because $y_{uv}$, corresponding to the best path between $\langle u, v \rangle$, should be no worse than $y_{uz} \bigotimes y_{zv}$, the best path between $\langle u, v \rangle$ going through $z$. Here, \(\cle\) is a generalized inequality relation that will be specified according to the type of PPM of interest. For example, when the performance metric is \textit{delay} (where better performance corresponds to lower values), \(\cle\) becomes \(\leq\). Thus, we will only punish the violation of the above generalized inequality, resulting in the following generalized ReLU style loss given $f_{\theta}(u,v)=\hat y_{u,v}$:
\begin{equation}\label{eq:bound}
    \min_{\theta, \tilde{A}} \mathcal{L}_{M}(\theta, \tilde{A}, Y) = \sum_{u,v\in V} l_{M}(f_{\theta}(u, v, \tilde{A}), y_{uz} \bigotimes y_{zv}),
    % ~\text{s.t.}~\tilde{A} \in \mathcal{A},
\vspace{-3pt}
\end{equation}
where \(z\) is a random node in \(p^{(n)}_{uv}\), and \(p^{(n)}_{uv}\) is the path with the optimal performance in \(\mathcal{P}_{uv}^L\). The function \(l_M\) computes a penalty enforcing that the predicted performance does not exceed the bounded value, defined as \(l_M(\hat{y}, y) = \max(0, \hat{y} \ominus y)\), where \(\ominus\) is also chosen adaptively based on the type of path performance metric. As a result, the estimated performance metric is always bounded by the optimal performance among the observed paths.

\subsection{Graph Structure Completion}\label{sec:topology}
Real-world networks are typically sparse, noisy, connected, and partially unobservable \citep{zhou2013learning, fan2017network}, making accurate topology learning essential for network tomography. The learned adjacency matrix must ensure reachability of observed node pairs while preserving sparsity and connectivity.
Given an observed network adjacency matrix \( A \), we formulate this process as a structure learning problem:
\begin{equation}\label{eq:sparse}
    \min_{\tilde{A}} \mathcal{L}_{S} = \|\mathcal{R}(M\odot\tilde A)-\mathcal{R}(A)\|_{F}^2 + S(\tilde{A}), \quad \text{s.t.}~\tilde{A} \in \mathcal{A},
\end{equation}
where \( \mathcal{R}(A) \) maps \( A \) to a binary matrix indicating the existence of paths between node pairs. \( M \in \{0,1\}^{|V|\times|V|} \) marks observed edge connectivity, and \( \|\cdot\|_F^2 \) ensures the learned adjacency matrix \( \tilde{A} \) preserves the observed pairwise reachability.
% The first term, \( \|\mathcal{R}(M\odot\tilde A)-\mathcal{R}(A)\|_{F}^2 \), enforces that the learned adjacency matrix \( \tilde{A} \) retains the pairwise reachability properties of \( A \). 
\( S(\tilde{A}) \) imposes constraints to enforce both connectivity and sparsity. To encourage sparsity, we minimize the \( \ell_1 \)-norm, which promotes a sparse adjacency structure \citep{candes2012exact, tsybakov2011nuclear}. 

For connectivity, although the network topology is unknown or partially incorrect, the existence of at least one path between observed node pairs ensures the graph is weakly connected. PPMs further guarantee path reachability, allowing us to infer structural properties and detect Strongly Connected Components (SCCs) by Tarjan’s Algorithm \citep{tarjan1972depth}. The challenge lies in embedding this prior knowledge into \hyperref[eq:sparse]{Equation} (\ref{eq:sparse}), as enforcing connectivity constraints is inherently discrete. To enable gradient-based optimization, we transform these constraints into a continuous form through a novel convex formulation that encodes both weak connectivity and SCCs, as detailed in \hyperref[th:connect]{Theorem} \ref{th:connect}.

\begin{theorem}[Directed Graph Connectivity Constraints]\label{th:connect}
Let \( G = (V, E) \) be a directed graph with a non-negative adjacency matrix \( A_G \in \mathbb{R}^{|V| \times |V|} \). For any graph with adjacency matrix \( A \in \mathbb{R}^{n \times n} \), define:
\[
Z_A = \operatorname{diag}\left(A \cdot \mathbf{1}^\top\right) - A + \frac{1}{|n|} \mathbf{1} \mathbf{1}^\top,
\]
where \( \mathbf{1} \in \mathbb{R}^{|n|} \) is an all-ones vector. Let \( \mathcal{G} \) denote the set of SCCs identified via PPMs. The entire graph is weakly connected, with each SCC \( g \in \mathcal{G} \) strongly connected, if and only if:
\begin{enumerate}
    \item \( Z_{A_G} \succ 0, \quad A \geq 0, \)
    \item \( Z_{A_g} + Z_{A_g}^\top \succ 0, \quad A_g \geq 0, \forall g \in \mathcal{G}, \)
\end{enumerate}
where \( A_g \) is the adjacency matrix of component \( g \).
\end{theorem}
The proof can be found in \hyperref[proof:connect]{Appendix} \ref{proof:connect}. 
% However, directly incorporating a connectivity constraint into the objective function introduces non-convexity, leading to optimization challenges such as multiple local minima \citep{ghosh2006growing, kumar2019structured}. While existing approaches enforce connectivity in undirected graphs using connected-component constraints \citep{fiedler1973algebraic, zhou2006learning}, network tomography requires connectivity in directed graphs, as the optimal path from node \( v \) to \( u \) does not necessarily imply an optimal path from \( u \) to \( v \).
% Here, we adopt the convex formulation as follows:
% \begin{equation}\label{eq:connect}
%     \operatorname{diag}((A + A^\top)/2 \cdot \mathbf{1}^\top) - (A + A^\top)/2 + \mathbf{1}^\top \mathbf{1}/|V| \succ 0,
% \end{equation}
% where \( \operatorname{diag}(x) \) denotes a diagonal matrix with elements from the vector \( x \), and \( \mathbf{1} \in \mathbb{R}^{1\times |V|} \) is an all-ones vector. The symbol \( \succ 0 \) indicates that the matrix is positive definite. This constraint prevents the formation of isolated nodes or disconnected subgraphs, ensuring that the learned structure remains meaningful in practical scenarios.
Incorporating the connectivity constraints from \hyperref[th:connect]{Theorem}~\ref{th:connect} and the sparsity constraints on the inferred adjacency matrices \(\tilde{A}\) and \(\tilde{A}_g\) for the graph and components \( g \in \mathcal{G} \), the topological constraint \( S(\tilde{A}) \) for DeepNT is derived, and \hyperref[eq:sparse]{Equation}~\ref{eq:sparse} is reformulated as:
\begin{align}\label{eq:sparse_final}
    \begin{gathered}
        \min_{\tilde{A}} \|\mathcal{R}(M \odot \tilde{A}) - \mathcal{R}(A)\|_{F}^2 + \alpha \|\tilde{A}\|_1, \\[-0.5em]
        \text{s.t.} \quad \operatorname{diag}(\tilde{A} \cdot \mathbf{1}^\top) - \tilde{A} + \mathbf{1}^\top \mathbf{1}/|V| \succ 0,~\tilde{A} \geq 0, \tilde{A} \in \mathcal{A} \\
        \operatorname{diag}(\tilde{A}_g \cdot \mathbf{1}^\top) - \operatorname{Sym}(\tilde{A}_g) + \mathbf{1}^\top \mathbf{1}/|g| \succ 0,~ \forall g \in \mathcal{G}, \tilde{A}_g \geq 0.
    \end{gathered}
\end{align}
where $\operatorname{Sym}(A) = (A + A^\top)/2$ means the symmetrization of matrix $A$, $\alpha$ controls the contribution from the sparsity constraint. 

\subsection{Optimization for DeepNT}
Based on the defined constraints, we formulate the optimization of DeepNT as a continuous optimization problem solvable via gradient descent. We jointly learn the GNN model and the adjacency matrix to infer the optimal network topology for the GNN model. The final objective function of DeepNT is given as,
\begin{align}\label{eq:loss}
    &\argmin_{\theta, \tilde{A}} \mathcal{L} = \mathcal{L}_{GNN} + \mathcal{L}_{S} + \gamma\mathcal{L}_{M}\\
    \text{s.t.}\quad \operatorname{diag}((\tilde{A} +& \tilde{A}^\top)/2 \cdot \mathbf{1}^\top) - (\tilde{A} + \tilde{A}^\top)/2 + \mathbf{1}^\top \mathbf{1}/|V| \succ 0,~ \tilde{A} \in \mathcal{A},\notag
\end{align}
where $\gamma$ is a predefined parameter. Jointly optimizing $\theta$ and $\tilde{A}$ is challenging because it involves navigating a highly non-convex optimization landscape with interdependent variables. The optimization problem in DeepNT is formulated as follows:
\begin{equation}\label{eq:loss}
    \mathcal{F}=\min_{\theta, \tilde{A}} g(\theta,\tilde{A})+ \alpha ||\tilde{A}||_1,
\end{equation}
where $g(\theta,\tilde{A})=\mathcal{L}_{GNN}  + \gamma\mathcal{L}_{M} +  \|\mathcal{R}(M\odot\tilde A)-\mathcal{R}(A)\|_{F}^2$. 
To facilitate effective learning, we introduce a proximal gradient algorithm with extrapolation, as detailed in \hyperref[alg:training]{Algorithm} \ref{alg:training}, where the graph is constrained to be weakly connected,
% $\omega>0$ is a learning rate, and 
$\operatorname{prox}_{\lambda \|\cdot\|_1}(f) = S_{\lambda}(f) = \arg\min\nolimits_x \left( \frac{1}{2} \|x-f\|_F^2 + \lambda \|x\|_1 \right)$ is the soft-thresholding operator. The algorithm initializes \(\theta\) and \(\tilde{A}\) (line 1), then iteratively updates parameters via gradient descent and optimizes \(\tilde{A}\) with proximal and connectivity constraints (lines 2-19) until convergence.

Convergence of our optimization algorithm is guaranteed by \hyperref[th1]{Theorem} \ref{th1}, whose proof is provided in \hyperref[proof]{Appendix} \ref{proof}.

\begin{theorem}\label{th1}
   Assume $g(\theta,\tilde{A})$ is Lipschitz continuous with coefficient $l>0$, and its gradient $\nabla g(\theta,\tilde{A})$ is Lipschitz continuous with coefficient $L>0$. Let $\frac{1}{L}\leq\omega\leq\sqrt{\frac{L}{L+l}}$, and let $\{(\theta^k,\tilde{A}^k)\}$ be a sequence generated by \hyperref[alg:training]{Algorithm} \ref{alg:training}, then any of its limit point $(\theta^*,\tilde{A}^*)$ is a stationary point of \hyperref[eq:loss]{Equation} \eqref{eq:loss}. 
\end{theorem}

We demonstrate that DeepNT's network tomography capabilities are guaranteed under sufficient observations, as established in \hyperref[thm:convergence]{Theorem}~\ref{thm:convergence}.

\begin{theorem}[Convergence of DeepNT Predictions to True Pairwise Metrics] \label{thm:convergence}
Let $G = (V, E)$. Suppose DeepNT is trained with an increasing number of observed node pairs $S \to T$, where $S \subseteq T = V \times V$, and the number of sampled paths $N \to \infty$. Then, the predicted pairwise metrics $\hat{y}_{uv} = f_{\text{DeepNT}}(u, v; \theta, \tilde{A})$ converge in expectation to the true metrics $y_{uv}$, i.e.,
\[
\lim_{S \to T} \lim_{N \to \infty} \mathbb{E}_{\langle u, v \rangle \sim T} \left[\lvert \hat{y}_{uv} - y_{uv} \rvert \right] = 0,
\]
\end{theorem}

The proof of \hyperref[thm:convergence]{Theorem} \ref{thm:convergence} is detailed in \hyperref[expressive]{Appendix} \ref{proof:convergence}.

\begin{algorithm}
\caption{Optimization of DeepNT}
\label{alg:training}
\begin{algorithmic}[1]
\Require Input data \(S\), labels \(y\), momentum parameter \(\omega\), graph components \(\mathcal{G}\)
\Ensure Parameters \(\theta\) of DeepNT, inferred adjacency matrix \(\tilde{A}\)

\State \(\tilde{A}^{-1} \gets 0\), \(\tilde{A}^0 \gets 0\), \(\theta^{-1} \gets 0\), \(\theta^0 \gets 0\)

\While{stopping condition is not met}
    \LineComment{Momentum-accelerated updates for parameters and adjacency matrix.}
    \State \(\overline{\theta}^{k} \gets \theta^k + (1 - \omega)(\theta^k - \theta^{k-1})\)
    \State \(\overline{A}^{k} \gets \tilde{A}^{k} + (1 - \omega)(\tilde{A}^{k} - \tilde{A}^{k-1})\)

    \LineComment{Gradient descent update for model parameters.}
    \State \(\theta^{k+1} \gets \overline{\theta}^k - \omega \nabla g(\overline{\theta}^k, \overline{A}^k)\)

    \LineComment{Proximal optimization for the global adjacency matrix.}
    \State \(\tilde{A}^{k+1} \gets \operatorname{prox}_{\omega \alpha \|\cdot\|_1}\left(\overline{A}^k - \omega \nabla g(\overline{\theta}^k, \overline{A}^k)\right)\)

    \LineComment{Enforce global connectivity constraint on \(\tilde{A}^{k+1}\).}
    \State \( \Delta \tilde{A} \gets \operatorname{diag}(\operatorname{Sym}(\tilde{A}^{k+1}) \cdot \mathbf{1}^\top) - \operatorname{Sym}(\tilde{A}^{k+1}) + \mathbf{1}^\top \mathbf{1}/|V| \)
    \State \(\tilde{A}^{k+1} \gets \operatorname{prox}_{\succ}(\Delta\tilde{A}) \)

    \LineComment{Proximal updates and connectivity enforcement.}
    \For{each \(g \in \mathcal{G}\)}
        \State \(\tilde{A}_g^{k+1} \gets \operatorname{prox}_{\omega \alpha \|\cdot\|_1}\left(\tilde{A}_g^k - \omega \nabla g(\overline{\theta}^k, \tilde{A}_g^k)\right)\)
        \State \( \Delta \tilde{A}_g \gets \operatorname{diag}(\operatorname{Sym}(\tilde{A}_g^{k+1}) \cdot \mathbf{1}^\top) - \operatorname{Sym}(\tilde{A}_g^{k+1}) + \mathbf{1}^\top \mathbf{1}/|g| \)
        \State \(\tilde{A}_g^{k+1} \gets \operatorname{prox}_{\succ}(\Delta \tilde{A}_g)\)
    \EndFor
\EndWhile
\end{algorithmic}
% \vspace{-20pt}
\end{algorithm}

\section{Experiment}\label{sec:exp}
% In this section, we evaluate the effectiveness of DeepNT and compare our approach with state-of-the-art network tomography methods in both PPM prediction and topology reconstruction. 
% % In addition to the performance in path performance metric prediction, we will also discuss the performance of DeepNT in topology reconstruction.

\subsection{Datasets}
We conduct experiments on three real-world datasets, encompassing transportation, social, and computer networks, each characterized by distinct path performance metrics. Transportation networks are collected from different cities. The social network dataset collects interactions between people on different online social platforms, including Epinions, Facebook and Twitter. The Internet dataset consists of networks with raw IPv6 or IPv4 probe data. Details on data statistics, data processing, and path performance metrics for each dataset can be found in \hyperref[sec:appendix:data]{Appendix} \ref{sec:appendix:data}.

We use a synthetic dataset to test the comprehensive performance of our model on networks of different sizes and properties, exploring the robustness and scalability of our model. Synthetic networks are generated using the Erdős-Rényi, Watts-Strogatz and Barabási–Albert models.
For network sizes in ${50i}_{i=1}^{50}$, each graph generation algorithm is used to generate 10 networks with varying edge probabilities for each network size (the edge probability represents the likelihood that any given pair of nodes in the network is directly connected by an edge). We focus on monitor-based sampling scenarios, where some nodes are randomly selected as monitors and the end-to-end path performance between the monitors and other nodes are sampled as training data. 
For all datasets, we set different sampling rates $\delta\in \{10\%, 20\% , 30\%\}$ to simulate the real network detection scenario \citep{ma2020neural}. The sampled path performance is used as training data, that is, $\delta$ of total node pairs are used as training data and the rest of node pairs are used for testing. 
% Details of implementation are provided in \hyperref[sec:appendix:imp]{Appendix} \ref{sec:appendix:imp}.

\begin{table*}[tb]
\centering
\caption{Mean Absolute Percentage Error (MAPE $\downarrow$) and Mean Squared Error (MSE $\downarrow$) for \textbf{Additive Metrics} on real-world datasets. The best results are highlighted in \textbf{bold}. The second best results are \underline{underlined}. $-$ indicates that the model either fails to handle the large network or requires an excessive amount of time to perform network tomography.}
\begin{tabular}{cc|cc|cc|cc|cc}
\hline
\multicolumn{1}{c}{\multirow{2}{*}{\textbf{$\frac{|S|}{|T|}$}}} & \multicolumn{1}{c|}{\multirow{2}{*}{\textbf{Method}}}      & \multicolumn{2}{c|}{\textbf{Internet}} & \multicolumn{2}{c|}{\textbf{Social Network}} & \multicolumn{2}{c|}{\textbf{Transportation}} & \multicolumn{2}{c}{\textbf{Synthetic}}                                          \\ \cline{3-10} 
\multicolumn{2}{c|}{}  & \textbf{MAPE} $\downarrow$ & \textbf{MSE} $\downarrow$ & \textbf{MAPE} $\downarrow$ & \textbf{MSE} $\downarrow$ & \textbf{MAPE} $\downarrow$ & \textbf{MSE} $\downarrow$ & \textbf{MAPE} $\downarrow$ & \textbf{MSE} $\downarrow$ \\ \hline
\multicolumn{1}{c|}{\multirow{9}{*}{10\%}} & MMP+DAIL  &  0.9411 & 143.9463 & - & - & 0.7642 & 41.0764 & 0.6755 & 318.3382 \\
\multicolumn{1}{c|}{}                      & BoundNT        & 0.9250 & 126.2529 & - & - & 0.7183 & 28.4639  & 0.6229  & 244.3805  \\
\multicolumn{1}{c|}{}                      & Subito        &  0.9368 & 129.8293  & - & - & 0.7809 & 39.9722 & 0.6047 & 236.2874  \\
\multicolumn{1}{c|}{}                      & PAINT        & 0.9337 & 130.6045 & - & - & \underline{0.6508} &  \underline{27.8604} &  \underline{0.3493} & 136.8139  \\
\multicolumn{1}{c|}{}                      & MPIP         & 0.9274 & 125.7296 & - & -  & 0.7294  & 28.0741 & 0.6246 & 253.4835  \\
\multicolumn{1}{c|}{}                      & NMF      & 0.9316 & 135.2296 & - & - & 0.7306 & 33.2205 &  0.5413 & 203.2034 \\
\multicolumn{1}{c|}{}                      & NeuMF       & 0.8431 & 112.0205 & 1.4390 & 25.1763 & 0.6732 & 28.9212 & 0.3925 & 151.6863 \\
\multicolumn{1}{c|}{}                      & NeuTomography & \underline{0.8118}  & \underline{97.0785}  & \underline{1.3872} & \underline{21.5816} & 0.6948  & 30.2343 & 0.3629  & \underline{133.9919} \\ \cdashline{2-10}
\multicolumn{1}{c|}{}                      & \textbf{DeepNT}        &  \textbf{0.6907}  &  \textbf{84.4514}  & \textbf{0.8172} & \textbf{12.6533}  & \textbf{0.6342} & \textbf{24.0135} &  \textbf{0.2520} &  \textbf{79.3843} \\ \hline
\multicolumn{1}{c|}{\multirow{9}{*}{20\%}} & MMP+DAIL     & 0.8892 & 124.0720 & - & - & 0.6982  & 32.4886   & 0.6324 & 261.3459 \\
\multicolumn{1}{c|}{}                      & BoundNT        & 0.8935 & 120.4519 & - & - & 0.6507 &  27.9272  & 0.5587  & 196.9912  \\
\multicolumn{1}{c|}{}                      & Subito        & 0.9008 & 122.5825 & - & - & 0.6757 & 30.7168  & 0.5571  & 194.0589 \\
\multicolumn{1}{c|}{}                      & PAINT         & 0.8638 & 112.4626 & - & - & \underline{0.5983} & \underline{25.1261}  & \underline{0.3091} &  \underline{113.6392}  \\
\multicolumn{1}{c|}{}                      & MPIP         & 0.8901 & 118.9798 & - & - & 0.6419  &  27.1587   &  0.5663  & 202.6942 \\
\multicolumn{1}{c|}{}                      & NMF      & 0.8790 & 118.0108 & - & - & 0.6714 & 31.2973 &  0.5175 & 186.9012 \\
\multicolumn{1}{c|}{}                      & NeuMF       & 0.7998 & 95.3064 & 1.3116 & 20.8148 & 0.6264 & 26.5191 & 0.3692 & 139.7447 \\
\multicolumn{1}{c|}{}                      & NeuTomography & \underline{0.7547} & \underline{90.1461}  &  \underline{1.2211} & \underline{18.1076}  & 0.6175 & 25.4259 & 0.3315 & 119.6274 \\ \cdashline{2-10}
\multicolumn{1}{c|}{}                      & \textbf{DeepNT}        &  \textbf{0.6299} &  \textbf{76.5168}  & \textbf{0.7593} &  \textbf{12.0193} & \textbf{0.5543} & \textbf{21.6331} &   \textbf{0.2168} &  \textbf{66.5215} \\ \hline
\multicolumn{1}{c|}{\multirow{9}{*}{30\%}} & MMP+DAIL     & 0.8219 & 104.2967 & - & - & 0.5839  & 28.1040  & 0.5702 & 218.5386 \\
\multicolumn{1}{c|}{}                      & BoundNT         & 0.8593 & 110.3346 & - & - & 0.5124  & 20.6403  &  0.4772  & 170.3272 \\
\multicolumn{1}{c|}{}                      & Subito         & 0.8466 & 107.0691 & - & - & 0.5493  & 20.6268  & 0.4966   & 169.6914  \\
\multicolumn{1}{c|}{}                      & PAINT         & 0.8108 & 102.0409 & - & - & 0.4629 &  20.0037  & \underline{0.2916}  & \underline{92.2561}  \\
\multicolumn{1}{c|}{}                      & MPIP         & 0.8532 & 109.7416 & - & - & 0.4905  & 20.1655   & 0.4712 & 171.0216   \\
\multicolumn{1}{c|}{}                      & NMF      & 0.8162 & 107.8148 & - & - & 0.5188 & 21.5137 & 0.4349 & 159.9384 \\
\multicolumn{1}{c|}{}                      & NeuMF       & 0.7513 & 88.0015 & 1.1774 & 16.8202 & 0.4652 & 19.8196 & 0.3308 & 122.0492 \\
\multicolumn{1}{c|}{}                      & NeuTomography &  \underline{0.7276}  & \underline{84.2087}  & \underline{1.1378}  &  \underline{16.3775}   & \underline{0.4433}  & \underline{19.1260}  & 0.3025 &  97.6045 \\ \cdashline{2-10}
\multicolumn{1}{c|}{}                      & \textbf{DeepNT}        &  \textbf{0.5842}  &  \textbf{71.0797} &  \textbf{0.7119}  & \textbf{10.8074} & \textbf{0.3794}  & \textbf{18.6551} &  \textbf{0.1935} &  \textbf{59.0406}       \\ \hline
\end{tabular}
\label{tab:main:add}
\end{table*}

\subsection{Evaluation}

\textbf{Comparison Methods.} To evaluate the effectiveness of DeepNT, we compare it with the state-of-the-art network tomography methods. Details of the implementation can be found in \hyperref[sec:appendix:imp]{Appendix}~\ref{sec:appendix:imp}.
 % Minimum Monitor Placement and Determination of All Identifiable Links 
\textbf{MMP+DAIL} \citep{ma2013identifiability} optimizes additive performance metrics under the assumption of a known network topology and manageable, loop-free routing. 
% Arbitrary-valued Non-additive Metric Identification 
\textbf{ANMI} \citep{ma2015optimal} locates problematic network links by employing a tunable threshold parameter and, given precise metric distributions, further estimates fine-grained link metrics. 
% Adaptive Measurements in Probabilistic Routing 
\textbf{AMPR} \citep{ikeuchi2022network} identifies network states in probabilistic routing environments by adaptively selecting measurements that maximize mutual information.
\textbf{BoundNT} \citep{feng2020bound} derives upper and lower bounds for unidentifiable links, using natural value bounds to constrain the solution space of the linear system. 
\textbf{Subito} \citep{tao2024delay} formulates a linear system and uses network tomography to estimate link delays with reinforcement learning. 
\textbf{PAINT} \citep{xue2022paint} iteratively estimates and refines link-level performance metrics, minimizing least square errors and discrepancies between estimated and observed shortest paths.
\textbf{MPIP} \citep{li2023bound} uses graph decomposition techniques and an iterative placement strategy to optimize monitor locations for improved inference of path metrics.
\textbf{NeuTomography} \citep{ma2020neural} learns the non-linear relationships between node pairs and the unknown underlying topological and routing properties by path augmentation and topology reconstruction.
Some studies formulate network tomography as a matrix factorization problem.
% Therefore, we also include \textbf{NMF} (Non-negative Matrix Factorization) \citep{lee2000algorithms} and \textbf{NeuMF} (Neural Matrix Factorization) \citep{he2017neural} as comparison methods.
Therefore, we also include \textbf{NMF} \citep{lee2000algorithms} and \textbf{NeuMF} \citep{he2017neural} as comparison methods.

\subsubsection{Main Results of Path Performance Metric Prediction}

The topological incompleteness is 0.2 (i.e., 20\% of the edges are replaced by non-existent edges). For the deep learning models, all experiments are performed 10 times and we report the average accuracy. For the linear system based methods, we adopt the solution from the authors’ original implementation. \hyperref[tab:main:add]{Table} \ref{tab:main:add}, Table \ref{tab:main:mul}, \hyperref[tab:main:minmax]{Table} \ref{tab:main:minmax} and \hyperref[tab:main:bool]{Table} \ref{tab:main:bool}  report the results of predicting additive, multiplicative, min/max and boolean path performance metrics, respectively. $\frac{|S|}{|T|}$ means how many node pairs’ end-to-end path performance metric values are measured and used for training.

\begin{table*}[tb]
\centering
% \vspace{-15pt}
\begin{minipage}{\linewidth}
\centering
% \captionsetup{font=scriptsize, skip=0.15mm} 
\caption{Mean Absolute Percentage Error (MAPE $\downarrow$) and Mean Squared Error (MSE $\downarrow$) for \textbf{Multiplicative Metrics} on real-world datasets. The best results are highlighted in \textbf{bold}. The second best results are \underline{underlined}. * indicates logarithmic transformations are used to convert multiplicative metrics to additive metrics, as these methods are designed for additive metrics.}
% \scriptsize
\begin{tabular}{cc|cc|cc|cc}
\hline
\multicolumn{1}{c}{\multirow{2}{*}{\textbf{$\frac{|S|}{|T|}$}}} & \multicolumn{1}{c|}{\multirow{2}{*}{\textbf{Method}}}      & \multicolumn{2}{c|}{\textbf{Internet}} & \multicolumn{2}{c|}{\textbf{Social Network}} & \multicolumn{2}{c}{\textbf{Synthetic}}                                          \\ \cline{3-8} 
\multicolumn{2}{c|}{}  & \textbf{MAPE} $\downarrow$ & \textbf{MSE} $\downarrow$ & \textbf{MAPE} $\downarrow$ & \textbf{MSE} $\downarrow$ & \textbf{MAPE} $\downarrow$ & \textbf{MSE} $\downarrow$ \\ \hline
\multicolumn{1}{c|}{\multirow{4}{*}{10\%}} & BoundNT (*)         & 0.3930 &  14.4673  & - & - & {0.0783} & {0.3651} \\
\multicolumn{1}{c|}{}                      & MPIP (*)        &  0.4007  &  16.4387  & - & - & {0.0791} & {0.3583} \\
\multicolumn{1}{c|}{}                      & NeuTomography & \underline{0.0632} & \underline{0.4216}  & \underline{0.0988} & \underline{0.0357} &  \underline{{0.0347}} & \underline{{0.0969}} \\ \cdashline{2-8}
\multicolumn{1}{c|}{}                      & \textbf{DeepNT}        & \textbf{0.0243} & \textbf{0.0438} &    
  \textbf{0.0620} &  \textbf{0.1247}  & \textbf{{0.0182}} & \textbf{{0.0154}} \\ \hline
\multicolumn{1}{c|}{\multirow{4}{*}{20\%}} & BoundNT (*)        & 0.3797 &  13.0014   & - & - & {0.0787} & {0.3301} \\
\multicolumn{1}{c|}{}                      & MPIP (*)        &  0.3835  &  12.5421   & - & - & {0.0803}  & {0.3498}  \\
\multicolumn{1}{c|}{}                      & NeuTomography &  \underline{0.0587}  & \underline{0.3493} &     \underline{0.0939} & \underline{0.0341}  & \underline{{0.0291}} & \underline{{0.0664}}  \\ \cdashline{2-8}
\multicolumn{1}{c|}{}                      & \textbf{DeepNT}       &  \textbf{0.0207} & \textbf{0.0257} & \textbf{0.0571} & \textbf{0.1094} & \textbf{{0.0136}} & \textbf{{0.0087}}  \\ \hline
\multicolumn{1}{c|}{\multirow{4}{*}{30\%}} & BoundNT (*)        &   0.3606  &  10.9892  & - & - & {0.0740} & {0.3125} \\
\multicolumn{1}{c|}{}                      & MPIP (*)        &  0.3588 &  12.8381   & - & - &  {0.0753} & {0.3216} \\
\multicolumn{1}{c|}{}                      & NeuTomography &  \underline{0.0526}  & \underline{0.2409}  &    \underline{0.0813} &  \underline{0.0226}  & \underline{{0.0243}} & \underline{{0.0381}} \\ \cdashline{2-8}
\multicolumn{1}{c|}{}                      & \textbf{DeepNT}     &   \textbf{0.0169} & \textbf{0.0093} & \textbf{0.0509}  & \textbf{0.0093}  &  \textbf{{0.0112}}   &  \textbf{{0.0083}}  \\ \hline
\end{tabular}
\label{tab:main:mul}
\end{minipage}

\vspace{2mm}
\begin{minipage}{\linewidth}
\centering
% \captionsetup{font=scriptsize, skip=0.15mm} 
\caption{Mean Absolute Percentage Error (MAPE $\downarrow$) and Mean Squared Error (MSE $\downarrow$) for \textbf{Min or Max Metrics} on real-world datasets. The best results are highlighted in \textbf{bold}. The second best results are \underline{underlined}.}
% \scriptsize
\begin{tabular}{cc|cc|cc|cc}
\hline
\multicolumn{1}{c}{\multirow{2}{*}{\textbf{$\frac{|S|}{|T|}$}}} & \multicolumn{1}{c|}{\multirow{2}{*}{\textbf{Method}}}      & \multicolumn{2}{c|}{\textbf{Internet}} & \multicolumn{2}{c|}{\textbf{Transportation}} & \multicolumn{2}{c}{\textbf{Synthetic}}                                          \\ \cline{3-8} 
\multicolumn{2}{c|}{}  & \textbf{MAPE} $\downarrow$ & \textbf{MSE} $\downarrow$ & \textbf{MAPE} $\downarrow$ & \textbf{MSE} ($\times 10^6$) $\downarrow$ & \textbf{MAPE} $\downarrow$ & \textbf{MSE} $\downarrow$ \\ \hline
\multicolumn{1}{c|}{\multirow{3}{*}{10\%}} & ANMI         & 0.0907  & 52.2783 & \underline{1.0674} & \underline{83.3370}  &  {0.0975} & {70.4885} \\
\multicolumn{1}{c|}{}                      & NeuTomography & \underline{0.0741} & \underline{37.8309}  & 1.1983 & 114.1017 &  \underline{{0.0770}}  & \underline{{51.1739}} \\ \cdashline{2-8}
\multicolumn{1}{c|}{}                      & \textbf{DeepNT}        & \textbf{0.0640} &  \textbf{34.4012} &  \textbf{0.4744}   & \textbf{38.1392} &  \textbf{{0.0585}} & \textbf{{29.7446}} \\ \hline
\multicolumn{1}{c|}{\multirow{3}{*}{20\%}} &  ANMI         & 0.0929 &        50.2317 &  1.1205  & 87.6417  & {0.0973} & {67.7634} \\
\multicolumn{1}{c|}{}                      & NeuTomography & \underline{0.0596} & \underline{28.8063} &  \underline{0.9016} &  \underline{73.8099}  & \underline{{0.0633}} & \underline{{33.1365}} \\ \cdashline{2-8}
\multicolumn{1}{c|}{}                      & \textbf{DeepNT}        &  \textbf{0.0517} &  \textbf{22.7196} &  \textbf{0.5216}  & \textbf{28.5838}  &  \textbf{{0.0431}} & \textbf{{21.7088}}  \\ \hline
\multicolumn{1}{c|}{\multirow{3}{*}{30\%}} & ANMI         & 0.0944  & 52.5456 & 1.0233  & 84.9310  & {0.0911} & {70.6701}  \\
\multicolumn{1}{c|}{}                      & NeuTomography &  \underline{0.0552} & \underline{21.2428} & \underline{0.8278} &  \underline{55.4812}  &  \underline{{0.0468}} & \underline{{18.2206}} \\ \cdashline{2-8}
\multicolumn{1}{c|}{}                      & \textbf{DeepNT}        & \textbf{0.0396}  &  \textbf{14.2014}  &  \textbf{0.4863} & \textbf{22.5173}  &  \textbf{{0.0332}}   &  \textbf{{12.9561}}  \\ \hline
\end{tabular}
\label{tab:main:minmax}
\end{minipage}

\vspace{2mm}
\begin{minipage}{\linewidth}
\centering
% \captionsetup{font=scriptsize, skip=0.15mm} 
\caption{Accuracy (ACC in \% $\uparrow$) and $F_1$ Score $\uparrow$ for \textbf{Boolean Metrics} on real-world datasets. The best results are highlighted in \textbf{bold}. The second best results are \underline{underlined}. $-$ means that the method cannot handle the network.}
% \scriptsize
\begin{tabular}{cc|cc|cc|cc}
\hline
\multicolumn{1}{c}{\multirow{2}{*}{\textbf{$\frac{|S|}{|T|}$}}} & \multicolumn{1}{c|}{\multirow{2}{*}{\textbf{Method}}}      & \multicolumn{2}{c|}{\textbf{Social Network}} & \multicolumn{2}{c|}{\textbf{Transportation}} & \multicolumn{2}{c}{\textbf{Synthetic}}                                          \\ \cline{3-8} 
\multicolumn{2}{c|}{}  & \textbf{ACC} $\uparrow$ & \textbf{$F_1$} $\uparrow$ & \textbf{ACC} $\uparrow$ & \textbf{$F_1$} $\uparrow$ & \textbf{ACC} $\uparrow$ & \textbf{$F_1$} $\uparrow$  \\ \hline
\multicolumn{1}{c|}{\multirow{3}{*}{10\%}}                       & AMPR          &  -  &  - & 0.7059 & 0.6184   & {0.6299} & {0.6178} \\
\multicolumn{1}{c|}{}                      & NeuTomography &  \underline{0.6429}  & \underline{0.6838}  &  \underline{0.7858}   & \underline{0.7493}  &  \underline{{0.6784}} & \underline{{0.7226}} \\ \cdashline{2-8}
\multicolumn{1}{c|}{}                      & \textbf{DeepNT}       & \textbf{0.6854} & \textbf{0.7122} & \textbf{0.8144}  & \textbf{0.8003} & \textbf{{0.7383}} &  \textbf{{0.7709}}  \\ \hline
\multicolumn{1}{c|}{\multirow{3}{*}{20\%}}                       & AMPR          &  - &  -  & 0.7517  & 0.6361    & {0.6497} & {0.6246} \\
\multicolumn{1}{c|}{}                      & NeuTomography &  \underline{0.6826} & \underline{0.7148}  & \underline{0.8092}  & \underline{0.7652} &  \underline{{0.6980}}  & \underline{{0.7837}} \\ \cdashline{2-8}
\multicolumn{1}{c|}{}                      & \textbf{DeepNT}       & \textbf{0.7045} & \textbf{0.7273} & \textbf{0.8317}  &  \textbf{0.8117}  & \textbf{{0.7726}}  & \textbf{{0.8163}} \\ \hline
\multicolumn{1}{c|}{\multirow{3}{*}{30\%}}                    &   AMPR          & - &  - &  0.7696 & 0.6605   & {0.6707} & {0.6422} \\
\multicolumn{1}{c|}{}                      & NeuTomography & \underline{0.7213}  & \underline{0.7484} & \underline{0.8551} & \underline{0.7795}  &  \underline{{0.7426}} & \underline{{0.7932}}  \\ \cdashline{2-8}
\multicolumn{1}{c|}{}                      & \textbf{DeepNT}       &  \textbf{0.7539} & \textbf{0.7691} &  \textbf{0.8784} & \textbf{0.8450} & \textbf{{0.8063}} & \textbf{{0.8361}} \\ \hline
\end{tabular}
\label{tab:main:bool}
\end{minipage}
\vspace{-12pt}
\end{table*}

For all types of path performance metrics, DeepNT consistently outperforms all other comparison methods. Low-rank approximation methods like NMF and NeuMF perform worse as observations decrease, as they depend on sufficient data (e.g., path performance across many node pairs) for reliable predictions. For additive metrics, although NeuTomography provides the second-best performance in both MAPE (0.8118) and MSE (97.0785) on Internet dataset at the 10\% sampling rate, DeepNT significantly outperforms it with a MAPE of 0.6907 and an MSE of 84.4514. The same pattern persists across the 20\% and 30\% sampling rates. DeepNT exhibits exceptional robustness when faced with different types of network structures (e.g., social, transportation, and synthetic networks), which current models are not well-equipped to handle. For multiplicative metrics, DeepNT’s performance remains stable across varying levels of network sparsity, as evidenced by its consistent top rankings in all scenarios. DeepNT achieves a MAPE of 0.0509 and an MSE of 0.0093 on large-scale social networks at sampling rate of 30\%, which is much better than NeuTomography’s MAPE of 0.0813 and MSE of 0.0226, while other models cannot handle these large-scale networks. 

\begin{figure*}[thb]
\centering
% Subfigure (a)
\begin{minipage}[t]{0.245\textwidth}
    \centering
    \includegraphics[width=\textwidth]{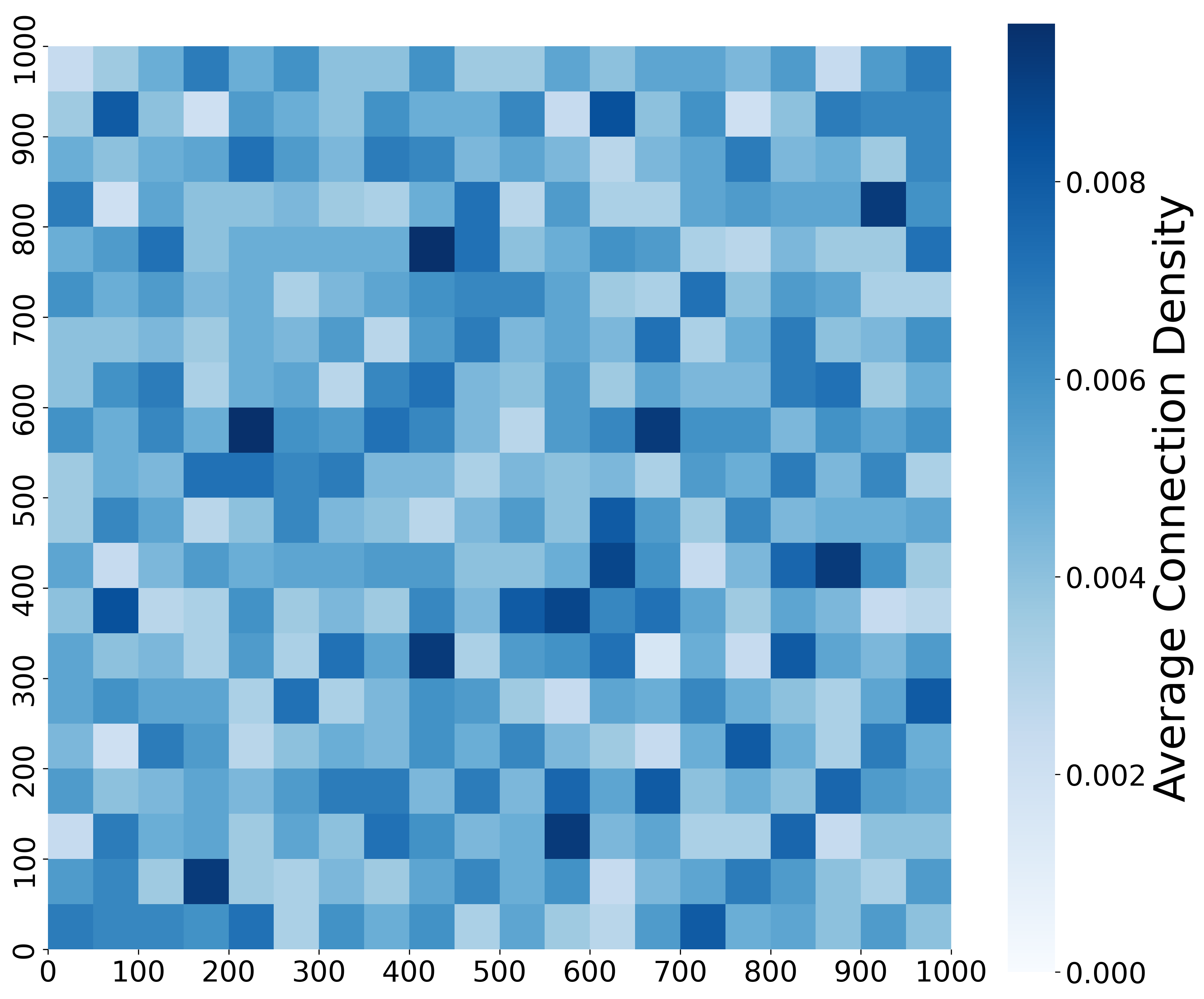}
    \small\textbf{(a)} Real topology $A$
    \label{fig:original}
\end{minipage}
% Subfigure (b)
\begin{minipage}[t]{0.245\textwidth}
    \centering
    \includegraphics[width=\textwidth]{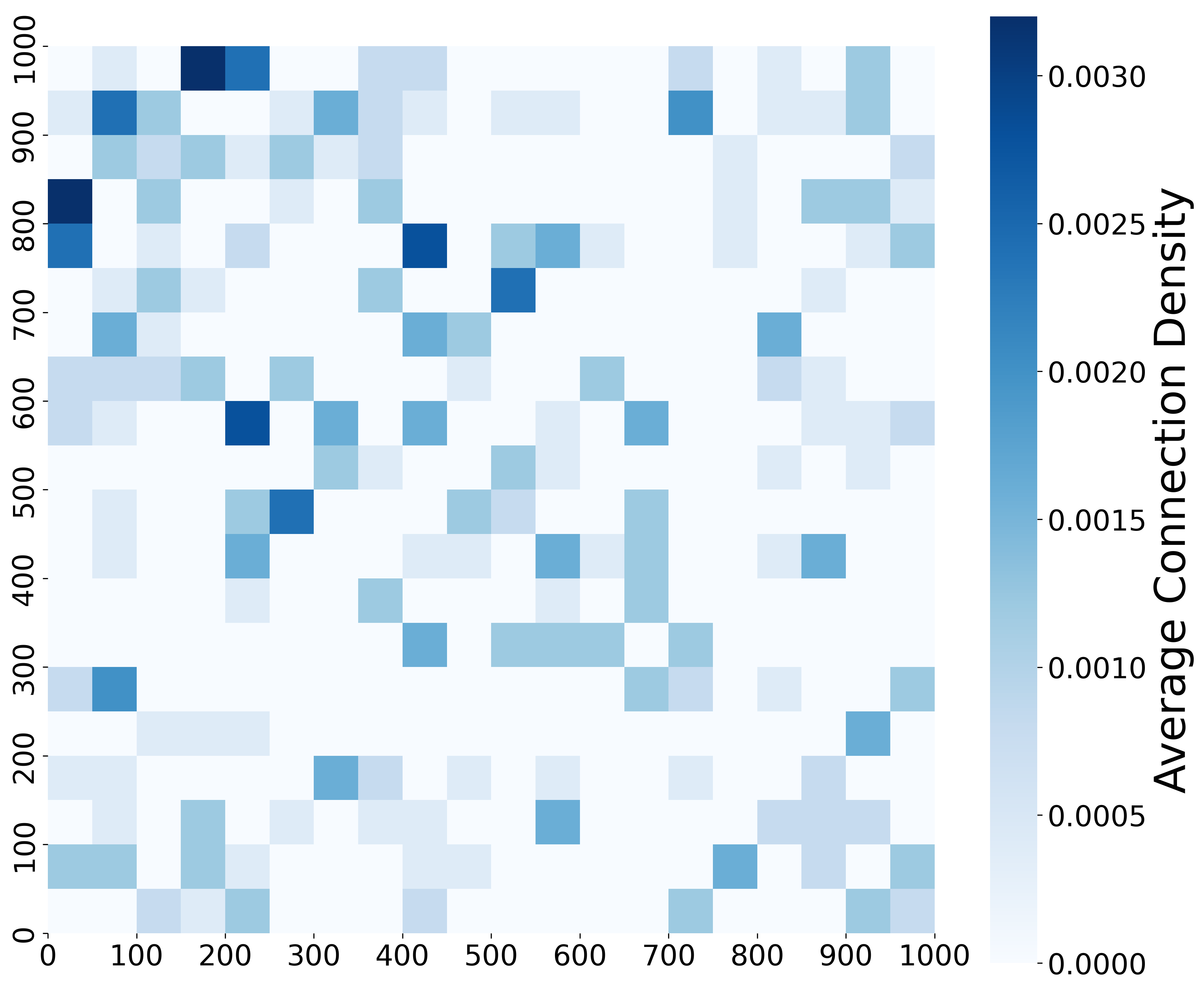}
    \small \textbf{(b)} $\Delta(A, A_{\text{Observed}})$
    \label{fig:observed}
\end{minipage}
% Subfigure (c)
\begin{minipage}[t]{0.245\textwidth}
    \centering
    \includegraphics[width=\textwidth]{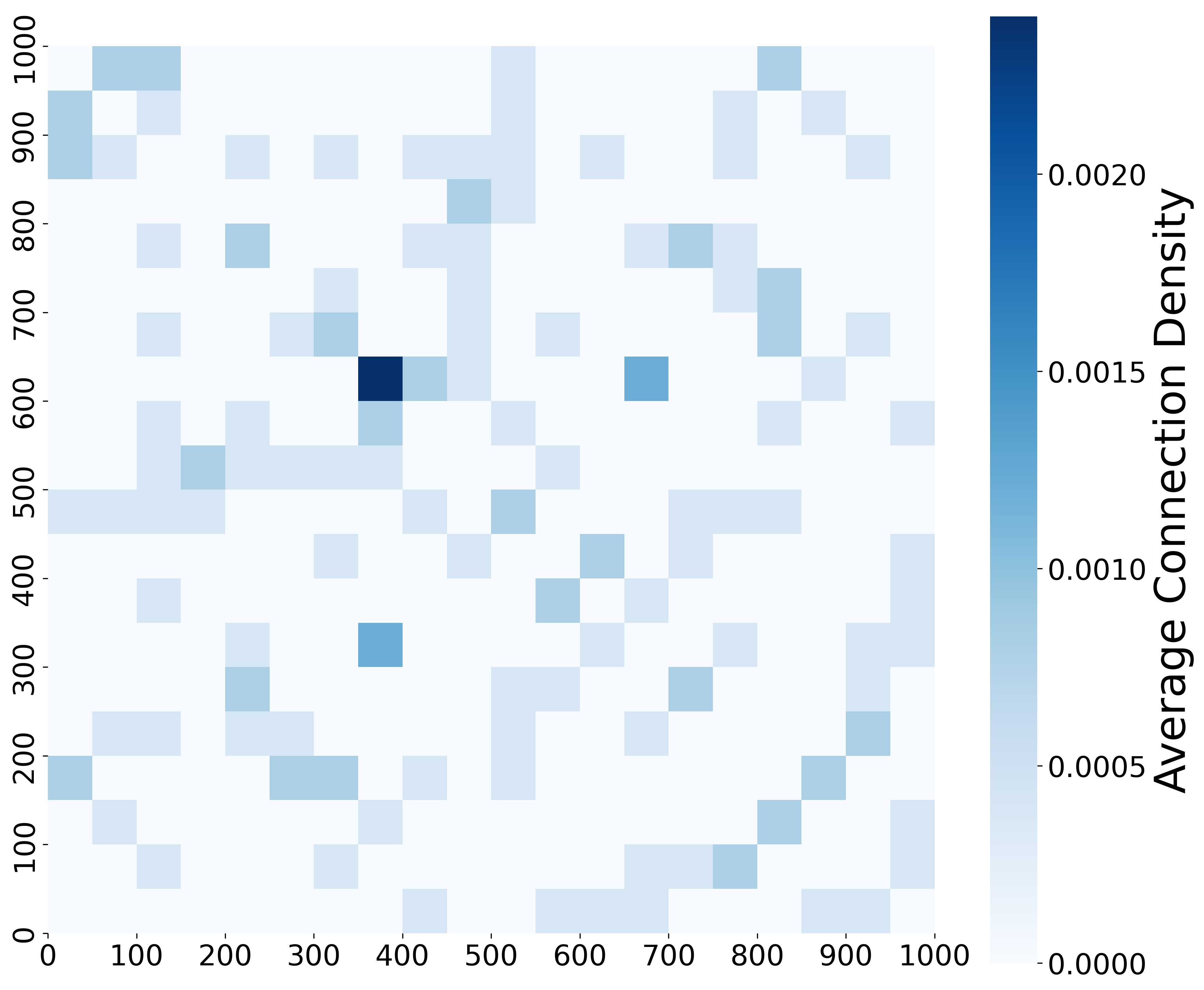}
    \small\textbf{(c)} $\Delta(A, \hat{A}_{\text{DeepNT}})$
    \label{fig:deepnt_topo}
\end{minipage}
% Subfigure (d)
\begin{minipage}[t]{0.245\textwidth}
    \centering
    \includegraphics[width=\textwidth]{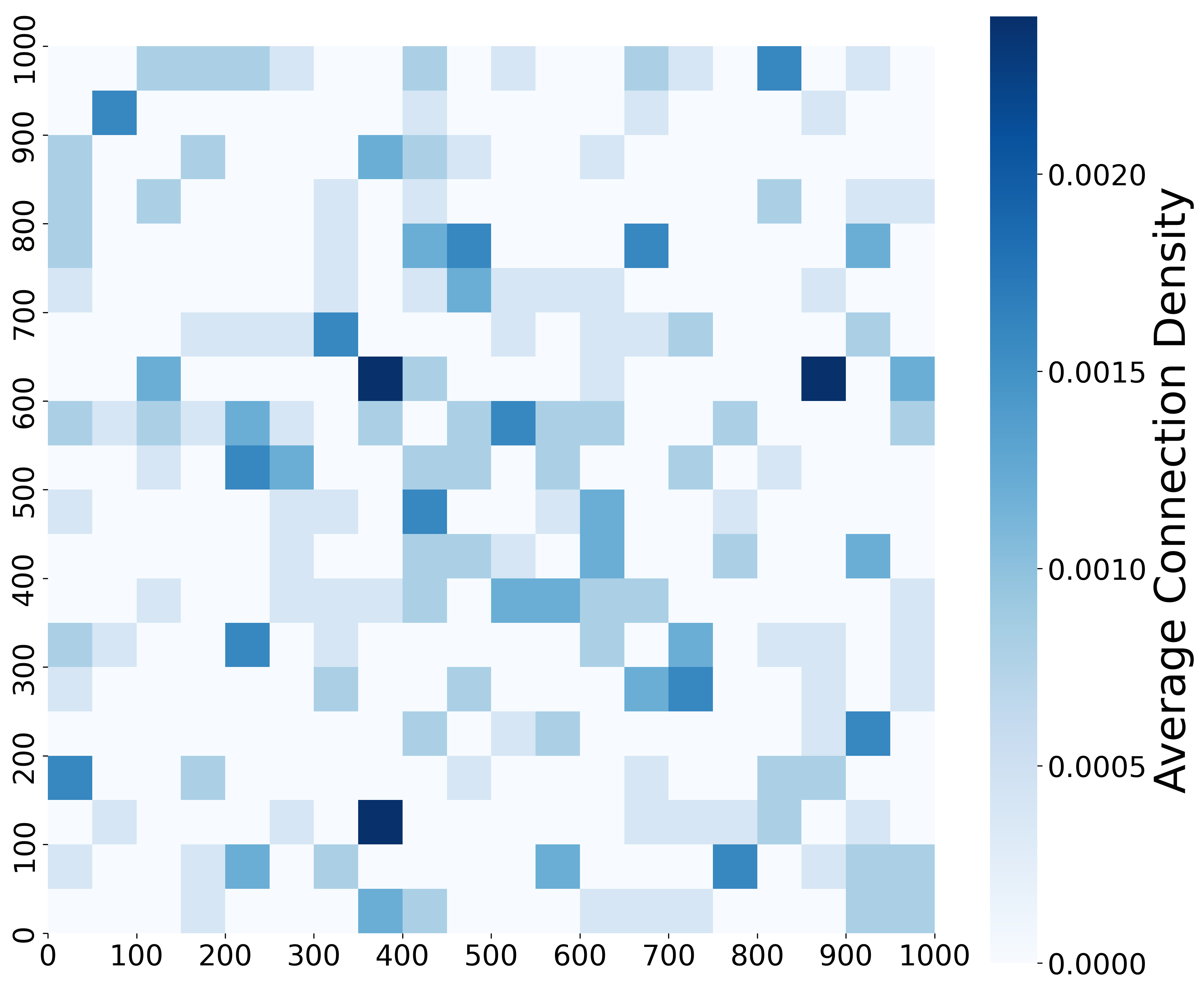}
    \small\textbf{(d)} $\Delta(A, \hat{A}_{\text{NeurTomography}})$
    \label{fig:neunt_topo}
\end{minipage}
\caption{Heatmap of the real adjacency matrix and the difference ($\Delta$) between the real and learned adjacency matrices from various models for a synthetic network with 1,000 nodes and 2,521 edges, considering a topological error rate of 0.2 and path performance metrics, such as bandwidth.}
\label{topology}
\vspace{-10pt}
\end{figure*}

% \textbf{Scalability Analysis.} 
As the network size increases, DeepNT shows superior scalability compared to the comparison models. For instance, for additional metrics on small datasets (e.g., the transportation dataset), comparison methods achieve comparable performance to DeepNT. At 10\% sampling, NeuTomography achieves a MAPE of 0.6948, close to DeepNT's 0.6342, while PAINT records a MAPE of 0.6508. This shows that on small networks, traditional models can be comparable to DeepNT. However, as the network size increases, such as on social network and Internet datasets, the performance gap between DeepNT and these comparison methods becomes obvious. On the Internet dataset at 30\% sampling rate, DeepNT achieves a MAPE of 0.5842, while NeuTomography lags behind with a MAPE of 0.7276. PAINT only achieves a MAPE of 0.8108 on the same dataset. 
% These results show that DeepNT scales better as the network size increases, maintaining low error rates even when traditional methods begin to fail.

\subsubsection{Case Study of Network Topology Reconstruction}
We further analyze the performance of network topology reconstruction given limited path information. We present a case study demonstrating the effectiveness of our proposed method for reconstructing network topology. For a network with 1,000 nodes and 2,521 edges with a topological error rate of 0.2, we visualize the heatmap of adjacency matrices where each block contains 50 nodes.

The path performance metric is the min/max metric, i.e., bandwidth, and $|S|/|T| = 30\%$. As shown in \hyperref[topology]{Figure} \ref{topology}, DeepNT successfully recovers most of the true topology with minimal deviation in topology reconstruction. The heatmaps in \hyperref[topology]{Figure} \ref{topology} demonstrate that the adjacency matrix learned by DeepNT is closer to the true adjacency matrix than the observed adjacency matrix and the adjacency matrix learned by NeuTomography. In particular, the denser and more complex parts of the network are more accurately recovered by DeepNT, which leads to smaller differences with the true adjacency matrix. 

\subsection{Ablation Study} 

To better understand how different components help our model predict various path performance metrics with incomplete network topology, we conduct ablation studies under different topology error rates $\Delta$ when $|S|/|T| = 30\%$. There are two key predefined parameters, i.e., $\alpha$ and $\gamma$, which control the contributions for sparsity and path performance bounds, respectively. We set the value of one parameter to one and the others to zero, and examine how the performance changes to show the impact of each component. 

Accordingly, two model variants, DeepNT-$\alpha$ and DeepNT-$\gamma$, are introduced. DeepNT-$\alpha$ sets $\alpha$ to $10^{-4}$ and $\gamma$ to 0, while DeepNT-$\gamma$ sets $\gamma$ to 1 and $\alpha$ to 0. We average the results of 500 {Erdős-Rényi} networks of the synthetic dataset for various tasks. \textit{Regression} represents the average results for predicting additive, multiplicative, and min/max path performance metrics, while \textit{Classification} reports the average results for predicting boolean path performance metrics. As shown in \hyperref[ablation]{Table} \ref{ablation}, when the sparsity ($\alpha$) or path performance bound ($\gamma$) constraints are removed, the performance significantly drops, demonstrating the importance of sparsity and boundary constraints under incomplete topological information. When the topology error rate $\Delta$ is small, DeepNT-$\gamma$ does not significantly improve the prediction performance. However, when $\Delta$ becomes larger, DeepNT-$\gamma$ (i.e., the path performance bound) can effectively reduce the impact of incorrect topology on prediction because it utilizes the possible correct path information to reconstruct the topology. Additionally, as $\Delta$ increases, the performance gap between DeepNT-$\alpha$ and DeepNT-$\gamma$ narrows, suggesting that maintaining sparsity in the adjacency matrix improves the model's performance lower bound. 

\begin{table}[ht]
\vspace{-10pt}
    \centering
    \captionof{table}{Ablation study results.}
    % \scriptsize
    \begin{tabular}{cl|c|ll}
    \hline
    \multicolumn{1}{c}{\multirow{2}{*}{$\Delta$}} & \multicolumn{1}{c|}{\multirow{2}{*}{\textbf{Method}}}      & \multicolumn{1}{c|}{\textbf{Regression}} & \multicolumn{2}{c}{\textbf{Classification}} \\ \cline{3-5} 
    \multicolumn{2}{c|}{}  & \textbf{MAPE} $\downarrow$ & \textbf{ACC} $\uparrow$ & \textbf{$F_1$} $\uparrow$ \\ \hline
    \multicolumn{1}{c|}{\multirow{3}{*}{10\%}}  & DeepNT  & 0.0741 & 0.8124 & 0.8043 \\ \cdashline{2-5}
    \multicolumn{1}{c|}{}  & DeepNT-$\alpha$ &  0.1236 &  0.6909 & 0.7341 \\ \cdashline{2-5}
    \multicolumn{1}{c|}{}  & DeepNT-$\gamma$ & 0.1014 & 0.7375  & 0.7582 \\ \hline
    \multicolumn{1}{c|}{\multirow{3}{*}{20\%}}  & DeepNT   &  0.0852 & 0.7701  & 0.8041 \\ \cdashline{2-5}
    \multicolumn{1}{c|}{}  & DeepNT-$\alpha$ &  0.1305 &  0.6628 & 0.6733 \\ \cdashline{2-5}
    \multicolumn{1}{c|}{}  & DeepNT-$\gamma$ &  0.1187 & 0.7081 & 0.7336\\ \hline
    \multicolumn{1}{c|}{\multirow{3}{*}{30\%}}  & DeepNT & 0.1129 & 0.7226 &  0.7636 \\ \cdashline{2-5}
    \multicolumn{1}{c|}{}  & DeepNT-$\alpha$ & 0.1368 &  0.6490 & 0.6827  \\ \cdashline{2-5}
    \multicolumn{1}{c|}{}  & DeepNT-$\gamma$ & 0.1212 & 0.6905& 0.7294 \\ \hline
    \end{tabular} 
    \label{ablation}
\end{table}
\vspace{-5pt}

% To further demonstrate the impact of removing constraints on performance, we present the distribution of ground truth and predicted values for the min/max path performance metric (bandwidth) in an IPv4 network with a topological error rate of 30\% and $|S|/|T| = 10\%$. As shown in Figure \ref{tradeoff}, even under a high topology error rate and limited path information, the probability distribution of predicted path performance metrics by DeepNT closely aligns with the true distribution, highlighting the model's effectiveness. Moreover, the predictions by DeepNT with constraints are closer to the true distribution compared to the variant without constraints, demonstrating the effectiveness of our training framework.

% \begin{figure}[htb]
%   \centering
%     \includegraphics[width=0.45\textwidth]{figs/ablation.png}
%     \captionof{figure}{Distribution of the ground truth and predicted min/max path performance metric value by DeepNT, i.e., bandwidth in an IPV4 network under $|S|/|T| = 10\%$.}
%   \label{tradeoff}
% \end{figure}

One of the key components of this paper is the proposed path aggregation layer, which we compare with the information aggregation layers currently widely used in various graph neural networks. Detailed results can be found in the \hyperref[app:agg]{Appendix} \ref{app:agg}. 
DeepNT with proposed path aggregation mechanism achieves the best performance for predicting additive metrics across most datasets.

\section{Conclusion}\label{sec:conclution}

In this paper, we introduce DeepNT, a novel framework for network tomography that addresses key challenges in predicting path performance metrics and network topology inference under incomplete and noisy observations.
% By leveraging Graph Neural Networks (GNNs), DeepNT learns path-centric representations and infers network structure. Our model introduces a constrained optimization objective, incorporating sparsity and connectivity constraints, which enhances the learning process by guiding the reconstruction of the adjacency matrix and bounding performance metrics through path performance bounds.
Through comprehensive experiments on real-world and synthetic datasets, DeepNT consistently outperforms state-of-the-art methods across a variety of path performance metrics, including additive, multiplicative, min/max, and boolean metrics. DeepNT demonstrates strong scalability and robustness, particularly as the network size and complexity increase, where traditional methods struggle to maintain performance.

\bibliographystyle{ACM-Reference-Format}
\bibliography{sample-base}

% \appendix


\appendix
\section{Appendix}
\subsection{Datasets} \label{sec:appendix:data}
The statistics of the real-world datasets are shown in \hyperref[tab:data]{Table} \ref{tab:data}. 
For the initialization representation of nodes, if the source data already contains node features, we use them as the initialization node representations, otherwise we use binary encoding to convert the node identifier (e.g., node index) into a binary representation. For path performance metrics, we use the edge labels of the source data to generate the path labels. In addition, we generate random edge labels of other path performance metric types for some networks, and then generate the path labels. Path performance metrics of each dataset are shown in the \hyperref[tab:data:metrics]{Table} \ref{tab:data:metrics}.

\begin{table*}[htb]
\centering
% \captionsetup{font=scriptsize, skip=0.15mm} 
\caption{Dataset Statistics: the number of networks, (average) nodes and edges. $-$ indicates the dataset has a singe network.}
\vspace{0.5mm}
% \scriptsize
\begin{tabular}{@{}l|cc|ccc|ccccc@{}}
\hline
{\multirow{2}{*}{\textbf{Statistics}}}      & \multicolumn{2}{c|}{\textbf{Internet}} & \multicolumn{3}{c|}{\textbf{Social Network}} & \multicolumn{5}{c}{\textbf{Transportation}} 
\\ \cline{2-11} 
\multicolumn{1}{c|}{}  & IPV4 & IPV6 & Epinions & Twitter & Facebook & Anaheim & Winnipeg & Terrassa & Barcelona & Gold Cost \\ \hline
Graphs & 10 & 10 & - & - & -  & - & - & - & - & -  \\ \cdashline{1-11}
Nodes & 2866.0 & 1895.7 & 75,879 & 81,306 & 4,039 & 416 & 1,057 & 1,609 & 1,020 & 4,807 \\ \cdashline{1-11}
Edges & 3119.6 & 2221.7 & 508,837 & 1768,149  & 88,234 &  914 & 2,535 & 3,264 & 2,522  & 11,140 \\ \hline
\end{tabular}
\label{tab:data}
\end{table*}

\begin{table*}[htb]
\centering
\caption{Properties of datasets. {\color{forestgreen}\cmark} of binary encoding indicates that the original data has no node features, and we use binary encoding to generate the initial node representation. {\color{burgundy}\xmark} means that binary encoding is not used, but the node features of the original data are used. For the path performance metrics, {\color{forestgreen}\cmark} for one metric indicates that the network has the true edge (link) labels of that metric, while {\color{flame}\cmark} indicates that random edge labels are generated for that performance metric.}
\begin{tabular}{@{}l|cc|ccc|ccccc@{}}
\hline
{\multirow{2}{*}{\textbf{Properties}}}      & \multicolumn{2}{c|}{\textbf{Internet\footnotemark[1]}} & \multicolumn{3}{c|}{\textbf{Social Network\footnotemark[2]}} & \multicolumn{5}{c}{\textbf{Transportation \footnotemark[3]}} 
\\ \cline{2-11} 
\multicolumn{1}{c|}{}  & IPV4 & IPV6 & Epinions & Twitter & Facebook & Anaheim & Winnipeg & Terrassa & Barcelona & Gold Cost \\ \hline
Binary Enc. & \color{forestgreen}\cmark & \color{forestgreen}\cmark & \color{burgundy}\xmark & \color{burgundy}\xmark & \color{burgundy}\xmark & \color{forestgreen}\cmark & \color{forestgreen}\cmark & \color{forestgreen}\cmark & \color{forestgreen}\cmark & \color{forestgreen}\cmark  \\ \hline
\multicolumn{11}{c}{{\fontfamily{pcr}\selectfont Additive Path Performance Metrics}}\\ \hline
Delay &  &  & \color{flame}\cmark  & \color{flame}\cmark  & \color{flame}\cmark  &  &  &  &  &  \\ \cdashline{1-11}
RTT & \color{forestgreen}\cmark & \color{forestgreen}\cmark &  &  &  &  &  &  &  &  \\ \cdashline{1-11}
Distance &  &  & \color{forestgreen}\cmark  & \color{forestgreen}\cmark  & \color{forestgreen}\cmark  &  &  &  &  &  \\ \cdashline{1-11}
Flow Time &  &  &  &  &  & \color{forestgreen}\cmark & \color{forestgreen}\cmark & \color{forestgreen}\cmark & \color{forestgreen}\cmark & \color{forestgreen}\cmark \\ \hline
\multicolumn{11}{c}{{\fontfamily{pcr}\selectfont Multiplicative Path Performance Metrics}}\\ \hline
Reliability & \color{flame}\cmark & \color{flame}\cmark &  &  &  &  &  &  &  &  \\ \cdashline{1-11}
Trust Decay &  &  & \color{flame}\cmark  & \color{flame}\cmark  & \color{flame}\cmark  &  &  &  &  &  \\ \hline
\multicolumn{11}{c}{{\fontfamily{pcr}\selectfont Min or Max Path Performance Metrics}}\\ \hline
Bandwidth & \color{flame}\cmark & \color{flame}\cmark &  &  &  &  &  &  &  &  \\ \cdashline{1-11}
Capacity &  &  &  &  &  & \color{forestgreen}\cmark & \color{forestgreen}\cmark & \color{forestgreen}\cmark & \color{forestgreen}\cmark & \color{forestgreen}\cmark \\ \hline
\multicolumn{11}{c}{{\fontfamily{pcr}\selectfont Boolean Path Performance Metrics}}\\ \hline
Is Trustworthy &  &  & \color{forestgreen}\cmark &  &  &  &  &  &  &  \\ \cdashline{1-11}
Is Secure &  &  &  &  &  & \color{flame}\cmark & \color{flame}\cmark & \color{flame}\cmark & \color{flame}\cmark & \color{flame}\cmark \\ \hline
\end{tabular}
\label{tab:data:metrics}
\end{table*}

\footnotetext[1]{\url{https://publicdata.caida.org/datasets/topology/ark}}
\footnotetext[2]{\url{https://snap.stanford.edu/data}}
\footnotetext[3]{\url{https://github.com/bstabler/TransportationNetworks}}

For the synthetic dataset, we utilize the Erdos-Renyi algorithm to generate networks of different sizes. Then, we generate random edge labels for all the above path performance metrics. When generating random edge labels for all datasets, for the addition and min/max metrics, the random labels are in $[1, 100]$, and for the multiplication metric, the random labels are in $[0.9, 0.999]$. For the Boolean metrics, each edge is randomly assigned a state of 0 or 1, where path labels are controlled to be balanced (the number of positive and negative labels will not less than 30\%).

\subsection{Implementation} \label{sec:appendix:imp}

The training data for each dataset depends on the sampling rate, i.e., $\frac{|S|}{|T|}$ which indicates how many node pairs’ end-to-end path performance metric values are used for training. Half of node pairs in $S$ is used as training data, and the other half is used as the validation set. That is, the number of node pairs actually used as training data is $\frac{|S|}{2}$. The rest of node pairs in \( T \setminus S \) are used as test data. 
% In our experimental setup, we test the prediction results of $\frac{|S|}{|T|} \in \{10\%, 20\%, 30\%\}$ for each network at different sampling rates and report the average results.

To train DeepNT, we use CrossEntropyLoss as the loss function for classification tasks (Boolean metric) and MSELoss as the loss function for regression tasks (additive, multiplicative, and min/max metrics). Adam optimizer is used to optimize the model. The learning rate is set to 1e-4 across all tasks and models. The training batch is set to 1024 and the test batch is 2048 for all datasets. We use GCN as the GNN backbone, and the number of layers of GCN is 2. We use the mean pooling as the READOUT function. An one-layer MLP is used to make predictions. All models are trained for a maximum of 500 epochs using an early stop scheme with the patience of 10. The hidden dimension is set to 256. The hyperparameters we tune include the number of sampled shortest paths $N$ in 1, 2, 3, the sparsity parameter $\alpha$ in 10e-5, 10e-4, 10e-3, 10e-2, and the path performance bound parameter $\gamma$ in 0.1, 0.25, 0.5, 1, 2, 4, 8, 16. 

For the comparison methods, we follow the original settings provided by the authors. In particular, the ANMI threshold ratio is reported as 30\%. AMPR requires multiple probe tests, and we set the number of probe tests to the number of placed monitors, since each monitor tests the end-to-end path performance with other nodes.

\subsection{Ablation Study on information aggregation}\label{app:agg}

We compare our path aggregation with aggregation operations from SEAL (subgraph aggregation) \citep{zhang2018link}, LESSR (edge-order preserving aggregation) \citep{chen2020handling}, PNA (principal neighborhood aggregation) \citep{corso2020principal}, and OSAN (ordered subgraph aggregation) \citep{qian2022ordered}. A brief description of these aggregation methods is provided in \hyperref[tab:graph_models]{Table} \ref{tab:graph_models}.

The results in \hyperref[tab:agg_comparison]{Table} \ref{tab:agg_comparison} demonstrate that the path aggregation mechanism in DeepNT achieves the lowest MSE for predicting additive metrics across most datasets, outperforming alternative aggregation methods. This suggests that path-specific attention effectively captures the most probable paths with optimal performance, leading to finer-grained and more accurate network tomography. Edge-order preserving aggregation performs best in the transportation dataset, likely due to its ability to preserve sequential dependencies in structured paths. However, in other cases, subgraph and ordered subgraph aggregation methods exhibit higher errors, indicating that local subgraph representations alone may not be sufficient for capturing global path information required in network tomography.

\begin{table}[ht]
    \centering
    \small
    \captionsetup{font=small, skip=1pt} 
    \caption{MSE results using different aggregation layers.}
    \begin{tabular}{lcccc}
        \hline
        \textbf{DeepNT} & \textbf{Internet} & \textbf{Social} & \textbf{Transportation} & \textbf{Synthetic} \\
        \hline
        w/ subgraph & 117.65 & 22.90 & 21.77 & 77.93 \\
        w/ edge-order & 84.25 & 14.19 & \textbf{17.11} & 62.65 \\
        w/ principal & 94.43 & 20.01 & 22.89 & 84.07 \\
        w/ ordered subg. & 101.73 & 17.26 & 20.57 & 71.14 \\
        w/ path (ours) & \textbf{71.08} & \textbf{10.81} & 18.66 & \textbf{59.04} \\
        \hline
    \end{tabular}
    \label{tab:agg_comparison}
\end{table}

\begin{table*}[tb]
    \centering
    \caption{Comparison of different graph aggregation models.}
    \begin{tabular}{llp{8cm}}
        \hline
        \textbf{Model} & \textbf{Formula} & \textbf{Description} \\
        \hline
        \multirow{3}{*}{DeepNT} &
        $\hat{h}_v^{(n)} = h_v + \sigma \left( \sum_{x \in \mathcal{P}_{vu}^{(n)}} \alpha_{vx}^{(n)} h_x^{(n)} \right)$ &
        \textbf{Path aggregation}: Aggregates over sampled optimal paths.  \\
        &
        $\alpha_{vx}^{(n)} = \operatorname{softmax} (r^T [h_v, h_x^{(n)}])$ &\\
        &
        $h_v = \operatorname{READOUT} (\hat{h}_v^{(n)} \cdot P_{vu}^{(n)} \subset P_{vu}^L)$ & \\
        \hline
        \multirow{2}{*}{SEAL} &
        $h_v = \operatorname{AGG}(h_v : v \in G_{h}^{x,y})$ &
        \textbf{Subgraph aggregation}: Aggregates node features within en- \\
        &
        $h_{vp} = \operatorname{READOUT}(h_v : G_{h}^{x,y} \in \text{enclosing subgraphs})$ & closing subgraphs extracted for target links.\\
        \hline
        \multirow{3}{*}{LESSR} &
        $h_v = \operatorname{GRU}(h_v, h_z : z \in \mathcal{N}(v))$ &
        \textbf{Edge-order preserving aggregation}: Combines GRU-based  \\
        &
        $h_{vp} = \operatorname{Attention}(h_v, h_x, \text{shortcut connections})$ & local aggregation with global attention from shortcut paths.\\
        &
        $h_v = \operatorname{READOUT}(\operatorname{EOPA}(h_v), \operatorname{SGAT}(h_v))$ & \\
        \hline
        \multirow{3}{*}{PNA} &
        $h_{agg} = \operatorname{AGG}(h_u : u \in \mathcal{N}(v))$ &
        \textbf{Principal neighbourhood aggregation}: Combines mean,  \\
        &
        $h_{\text{scaled}} = \operatorname{Scaler}(\operatorname{deg}(v)) \cdot h_{\text{agg}}$ & max, min, and std aggregators with degree-based scaling.\\
        &
        $h_v = \operatorname{Combine}(h_{\text{scaled}})$ & \\
        \hline
        \multirow{2}{*}{OSAN} &
        $h_s = \operatorname{AGG} (f_w(v) : v \in s)$ &
        \textbf{Ordered subgraph aggregation}: Aggregates features from  \\
        &
        $h_v = \operatorname{READOUT} (h_s : v \in s, s \in S)$ & WL-labeled subgraphs containing the node.\\
        \hline
    \end{tabular}
    \label{tab:graph_models}
\end{table*}

\subsection{Expressiveness Study}\label{expressive}

A path from a source node $v$ to a target node $u$ is denoted by $p_{uv} = [v_1, v_2, \dots, v_k]$, 
where $v_1 = v$, $v_k = u$, and $(v_i, v_{i+1}) \in E$ for $i \in \{1, \dots, k-1\}$. 
Paths contain distinct vertices, and the length of the path is given by $|p_{uv}| = k-1$, 
defined as the number of edges it contains. In this work, we consider paths that adhere to these criteria.

In practice, we only consider paths up to a fixed length $L$. 
Let $\mathcal{P}_{uv}^L$ denote the set of the sampled top-$N$ optimal paths from $u$ to $v$, selected based on the best path performance, with lengths not exceeding $L$.
Recall that $S$ and $T$ represent the set of node pairs with observed path performance and all possible node pairs, respectively. 
Define $\mathcal{SP} = \bigcup_{<u,v> \in S} \mathcal{P}_{uv}^L$ as the collection of all sampled paths, 
and let $\mathcal{AP}$ denote the collection of all paths between the node pair combinations in $T$.
We have $\mathcal{P}_{uv}^L \subset \mathcal{SP} \subseteq \mathcal{AP}$, where $\mathcal{SP} \to \mathcal{AP}$ as $N \to \infty$ and $S \to T$. To strengthen the proof, we first introduce the concepts of WL-Tree and Path-Tree as defined by \citeauthor{michel2023path}.

\begin{definition}[WL-Tree rooted at $v$]
    Let $G = (V, E)$. A WL-Tree $W_v^L$ is a tree rooted at node $v \in V$ encoding the structural information captured by the 1-WL algorithm up to $L$ iterations. At each iteration, the children of a node $u$ are its direct neighbors, $\mathcal{N}(u) = \{w \mid (u, w) \in E\}$. 
    % Nodes are labeled with colors assigned by the 1-WL algorithm, which represent their equivalence within the $k$-hop neighborhood. 
\end{definition}

\begin{definition}[Path-Tree rooted at $v$]
    Let $G = (V, E)$. A Path-Tree $P_v^L$ rooted at a node $v \in V$ is a tree of height $L$, where the node set at level $k$ is the multiset of nodes that appear at position $k$ in the paths of $\mathcal{P}_v^L$, i.e., $\{ u \mid p^L(k) = u \text{ for } p^L \in \mathcal{P}_v^L \}.$ Nodes at level $k$ and level $k+1$ are connected if and only if they occur in adjacent positions $k$ and $k+1$ in any path $p^L \in \mathcal{P}_v^L$.
\end{definition}

\begin{theorem}[DeepNT-$\mathcal{AP}$ Expressiveness Beyond 1-WL]\label{th2}
Let $G = (V, E)$, $W_v^L$ and $W_u^L$ denote the WL-Trees of height $L$ rooted at nodes $v, u \in V$, respectively. Let $f_{\text{DeepNT}}(v)$ and $f_{\text{DeepNT}}(u)$ represent the embeddings produced by DeepNT when it has access to the complete set of paths $\mathcal{AP}$. Then the following holds:
\begin{enumerate}
    \item If $W_v^L \neq W_u^L$, then $f_{\text{DeepNT}}(v) \neq f_{\text{DeepNT}}(u)$.
    \item If $W_v^L = W_u^L$, it is still possible that $f_{\text{DeepNT}}(v) \neq f_{\text{DeepNT}}(u)$.
\end{enumerate}
\end{theorem}
\begin{proof}
To prove this statement, we refer to Theorem 3.3 from \cite{michel2023path}, which states that if $W_v^L$ is structurally different from $W_u^L$ (i.e., not isomorphic), then $P_v^L$ is also structurally different from $P_u^L$. Moreover, $P_v^L$ and $P_u^L$ can still differ even if $W_v^L$ and $W_u^L$ are identical.

We first address the case where $W_v^L \neq W_u^L$. It follows that $P_v^L$ and $P_u^L$ will not be isomorphic. 
The path aggregation layer in DeepNT is straightforward to prove as injective, as it employs a permutation-invariant readout function.
Consequently, DeepNT aggregates path-centric structural information from $P_v^L$ and $P_u^L$ to produce embeddings $f_{\text{DeepNT}}(v)$ and $f_{\text{DeepNT}}(u)$, which are guaranteed to be distinct.

Now consider the case where $W_v^L = W_u^L$. This implies that 1-WL cannot distinguish between $v$ and $u$. However, $P_v^L$ and $P_u^L$ may still differ. The path aggregation layer of DeepNT captures path-centric structural information from $P_v^L$ and $P_u^L$, resulting in distinct embeddings $f_{\text{DeepNT}}(v)$ and $f_{\text{DeepNT}}(u)$. Thus, DeepNT surpasses the expressiveness of 1-WL. This completes the proof.
\end{proof}

Finally, \hyperref[th3]{Theorem}~\ref{th3} can be readily proved by noting that the node pairs $\langle u, v \rangle$ and $\langle u', v' \rangle$ are represented by concatenated node embeddings. Since $\mathcal{P}_{uv}^L \neq \mathcal{P}_{u'v'}^L$, the distinctiveness of these representations is ensured by the expressive power of DeepNT, which surpasses that of 1-WL.

\subsection{Proof of \hyperref[th:connect]{Theorem} \ref{th:connect}}\label{proof:connect}
% \textcolor{red}{Pending Review.} 
% We extend Theorem 1 in \citep{zhao2019spatial} to directed graphs, originally formulated for undirected graphs.
\begin{proof}
For the weak connectivity of the entire graph, following \citeauthor{zhao2019spatial}, we show that for any nonzero vector \( x \in \mathbb{R}^{|V|} \):
\[
x^\top Z_G x = \sum_{i \neq j} A_{ij}(x_i - x_j)^2 + \frac{1}{|V|} \left( \sum_{i=1}^{|V|} x_i \right)^2 > 0.
\]
The first term ensures that differences between connected nodes contribute positively, while the rank-1 term \( \frac{1}{|V|} \left( \sum_{i=1}^{|V|} x_i \right)^2 \) prevents the existence of isolated components when edge directions are ignored. This guarantees that the graph \( G \) is weakly connected.

For strong connectivity within each component \( g \in \mathcal{G} \), we establish both necessity and sufficiency. For necessity, if \( g \) is strongly connected, then its adjacency submatrix \( A_g \) is irreducible by the Perron--Frobenius theorem, implying that \( A_g + A_g^\top \) provides bidirectional coupling between any partition of nodes in \( g \). Combined with the positive diagonal terms \( \operatorname{diag}(A_g \cdot \mathbf{1}^\top) + \operatorname{diag}(A_g^\top \cdot \mathbf{1}^\top) \) and the rank-1 term \( 2\left(\frac{1}{|g|}\mathbf{1}\mathbf{1}^\top\right) \), this ensures \( Z_g + Z_g^\top \succ 0 \).

For sufficiency, we prove by contradiction. If \( g \) were not strongly connected, then \( A_g \) would be permutable to a block upper-triangular form as
\(
A_g = \begin{pmatrix} A_{11} & A_{12} \\ 0 & A_{22} \end{pmatrix},
\)
where \( A_{11} \) and \( A_{22} \) correspond to disconnected subcomponents. We could construct a nonzero vector \( y \) conforming to these blocks such that \( y^\top (Z_g + Z_g^\top) y = 0 \), contradicting the positive definiteness of \( Z_g + Z_g^\top \). Therefore, \( Z_g + Z_g^\top \succ 0 \) if and only if \( g \) is strongly connected.
This completes the proof.
\end{proof}

\subsection{Proof of \hyperref[th1]{Theorem} \ref{th1}}\label{proof}
\begin{proof}[Proof Sketch]
The coercivity of the objective function \( g(\theta, \tilde{A}) + \alpha \|\tilde{A}\|_1 \) ensures that the sequence \( \{(\theta^k, \tilde{A}^k)\} \) is bounded. By the Bolzano–Weierstrass theorem, there exists at least one limit point \( (\theta^*, \tilde{A}^*) \) \cite{wen2017linear}. Since the step size satisfies \( \frac{1}{L} \leq \omega \leq \sqrt{\frac{L}{L + l}} \), the function sequence \( \{\mathcal{F}(\theta^k, \tilde{A}^k)\} \) is non-increasing and converges to a finite value. The proximal gradient updates ensure \( \|\nabla g(\theta^{k+1}, \tilde{A}^{k+1}) - \nabla g(\theta^k, \tilde{A}^k)\| \to 0 \), implying that the gradients converge. Taking the limit of the update steps, we obtain \( \nabla_\theta g(\theta^*, \tilde{A}^*) = \mathbf{0} \), and the optimality condition for \( \tilde{A} \) satisfies \( \mathbf{0} \in \nabla_{\tilde{A}} g(\theta^*, \tilde{A}^*) + \alpha \partial \|\tilde{A}^*\|_1 \). Since every limit point of \( \{(\theta^k, \tilde{A}^k)\} \) satisfies these conditions, the sequence converges to a stationary point, completing the proof.
\end{proof}

\begin{proof}
    % The existence of a limit point $(\theta^*,\tilde{A}^*)$ can be derived from Proposition 3.5 in \cite{wen2017linear}, and $(\theta^*,\tilde{A}^*)$ is a stationary point of \eqref{eq:loss} by Lemma 3.4 in \cite{wen2017linear}.

% The coercivity of \( g(\theta, \tilde{A}) + \alpha \|\tilde{A}\|_1 \) ensures that the sequence \( \{(\theta^k, \tilde{A}^k)\} \) is bounded. 
% Hence, there exists at least one limit point \( (\theta^*, \tilde{A}^*) \). 
% By the Bolzano-Weierstrass theorem, every bounded sequence in \(\mathbb{R}^n\) has a convergent subsequence.
% Hence, the sequence \(\{(\theta^k, \tilde{A}^k)\}\) has at least one limit point \((\theta^*, \tilde{A}^*)\).

With the chosen step size satisfying \( \frac{1}{L} \leq \omega \leq \sqrt{\frac{L}{L + l}} \), the proximal gradient algorithm ensures:
\[
\mathcal{F}(\theta^{k+1}, \tilde{A}^{k+1}) \leq \mathcal{F}(\theta^k, \tilde{A}^k).
\]
This implies that the sequence \( \{\mathcal{F}(\theta^k, \tilde{A}^k)\} \) is non-increasing. Since \( \mathcal{F}(\theta, \tilde{A}) \) is bounded below (due to the coercivity of the \( \ell_1 \)-regularization term \( \alpha \|\tilde{A}\|_1 \)), the sequence \( \{\mathcal{F}(\theta^k, \tilde{A}^k)\} \) converges to a finite value. 

The boundedness of \( \mathcal{F}(\theta^k, \tilde{A}^k) \) ensures that the sequence \( \{(\theta^k, \tilde{A}^k)\} \) is bounded, which means \( \{(\theta^k, \tilde{A}^k)\} \) has at least one limit point \( (\theta^*, \tilde{A}^*) \).

Since \(\|(\theta^{k+1}, \tilde{A}^{k+1}) - (\theta^k, \tilde{A}^k)\| \to 0\), and \(\nabla g\) is Lipschitz continuous, it follows that:
\[
\|\nabla g(\theta^{k+1}, \tilde{A}^{k+1}) - \nabla g(\theta^k, \tilde{A}^k)\| \to 0.
\]
Therefore, the gradients \(\nabla g(\theta^k, \tilde{A}^k)\) converge to \(\nabla g(\theta^*, \tilde{A}^*)\) as \(k \to \infty\).

The update for \(\theta\) in \hyperref[alg:training]{Algorithm} \ref{alg:training} is:
\[
\theta^{k+1} - \overline{\theta}^k = -\omega \nabla_\theta g(\overline{\theta}^k, \overline{A}^k).
\]
% where \(\overline{\theta}^k = \theta^k + (1 - \omega)(\theta^k - \theta^{k-1})\).
As \(k \to \infty\), \(\overline{\theta}^k \to \theta^*\), \(\theta^{k+1} - \overline{\theta}^k \to 0\) and \(\nabla_\theta g(\overline{\theta}^k, \overline{A}^k) \to \nabla_\theta g(\theta^*, \tilde{A}^*)\). Then, it follows that:
\[
\nabla_\theta g(\theta^*, \tilde{A}^*) = \mathbf{0}.
\]
% This satisfies the optimality condition \eqref{eq:theta_optimality}.
The update for \( \tilde{A} \) in \hyperref[alg:training]{Algorithm} \ref{alg:training} involves solving the proximal operator:
\[
\tilde{A}^{k+1} = \arg\min_{\tilde{A}} \left( \frac{1}{2} \left\| \tilde{A} - \left( \overline{A}^k - \omega \nabla_{\tilde{A}} g(\overline{\theta}^k, \overline{A}^k) \right) \right\|_F^2 + \omega \alpha \|\tilde{A}\|_1 \right).
\]
% \mathcal{R}(\cdot) does not destroy the Lipschitz continuity (or employ subgradients if necessary).
This optimization is equivalent to applying the proximal mapping:
\[
\tilde{A}^{k+1} = \operatorname{prox}_{\omega \alpha \|\cdot\|_1}\left( \overline{A}^k - \omega \nabla_{\tilde{A}} g(\overline{\theta}^k, \overline{A}^k) \right),
\]
where \( \operatorname{prox}_{\lambda \|\cdot\|_1}(f) = S_{\lambda}(f) \) is the soft-thresholding operator. The proximal mapping satisfies the optimality condition:
\[
\mathbf{0} \in \tilde{A}^{k+1} - \left( \overline{A}^k - \omega \nabla_{\tilde{A}} g(\overline{\theta}^k, \overline{A}^k) \right) + \omega \alpha \partial \|\tilde{A}^{k+1}\|_1.
\]
Rearranging this condition gives:
\[
\mathbf{0} \in \nabla_{\tilde{A}} g(\overline{\theta}^k, \overline{A}^k) + \frac{1}{\omega} (\tilde{A}^{k+1} - \overline{A}^k) + \alpha \partial \|\tilde{A}^{k+1}\|_1.
\]
As \( k \to \infty \), the extrapolated sequence \( \overline{A}^k \to \tilde{A}^* \) and the proximal updates \( \tilde{A}^{k+1} \to \tilde{A}^* \). Consequently, the term \( (\tilde{A}^{k+1} - \overline{A}^k)/\omega \to \mathbf{0} \). Thus, the limit point \( \tilde{A}^* \) satisfies:
\[
\mathbf{0} \in \nabla_{\tilde{A}} g(\theta^*, \tilde{A}^*) + \alpha \partial \|\tilde{A}^*\|_1.
\]
We conclude that \( (\theta^*, \tilde{A}^*) \) is a stationary point of the optimization problem since both optimality conditions are satisfied:
\[
\mathbf{0} \in \nabla_\theta g(\theta^*, \tilde{A}^*), \quad
\mathbf{0} \in \nabla_{\tilde{A}} g(\theta^*, \tilde{A}^*) + \alpha \partial \|\tilde{A}^*\|_1.
\]
The algorithm may adjust \(\tilde{A}^{k+1}\) to ensure connectivity. This adjustment does not violate convergence guarantees because it is a bounded perturbation that preserves the descent property.

Therefore, the sequence \( \{(\theta^k, \tilde{A}^k)\} \) converges to the stationary point \( (\theta^*, \tilde{A}^*) \):
\(
\lim_{k \to \infty} (\theta^k, \tilde{A}^k) = (\theta^*, \tilde{A}^*).
\)
This establishes the convergence of the algorithm and completes the proof.

\end{proof}

\subsection{Proof of \hyperref[thm:convergence]{Theorem} \ref{thm:convergence}}\label{proof:convergence}

\begin{proof}
As \(N \to \infty\), the sampled paths \(\mathcal{SP}\) converge to the complete set of paths for the observed node pairs \(S\). As \(S \to T\), the set of observed node pairs expands to include all possible node pairs in \(T = V \times V\). Therefore, \(\mathcal{SP} \to \mathcal{AP}\) as \(N \to \infty\) and \(S \to T\), which means DeepNT-\(\mathcal{SP}\) converges to DeepNT-\(\mathcal{AP}\).

By \hyperref[th2]{Theorem}~\ref{th2} (DeepNT-$\mathcal{AP}$ Expressiveness Beyond 1-WL) and \hyperref[th3]{Theorem}~\ref{th3} (DeepNT-$\mathcal{AP}$ Distinguishes Node Pairs Beyond 1-WL), DeepNT-$\mathcal{AP}$ can uniquely identify and distinguish all node pairs based on differences in their path sets. Moreover, by \hyperref[th1]{Theorem}~\ref{th1} (Convergence of DeepNT's Optimization), the optimization of DeepNT converges to a stationary point \((\theta^*, \tilde{A}^*)\). This ensures that the empirical loss over the observed pairs \(S\) is minimized:
\[
\mathcal{L}(\theta^*, \tilde{A}^*) = \sum_{\langle u, v \rangle \in S} l(f_{\text{DeepNT}}(u, v; \theta^*, \tilde{A}^*), y_{uv}),
\]
where \(l(\cdot, \cdot)\) measures the error between the predicted metrics \(\hat{y}_{uv}\) and the true metrics \(y_{uv}\). Consequently, as \(S \to T\), the training error diminishes:
\[
\mathbb{E}_{\langle u, v \rangle \sim S}[\lvert \hat{y}_{uv} - y_{uv} \rvert] \to 0.
\]

As \(S \to T\), the observed set \(S\) becomes dense, covering all node pairs in \(T = V \times V\). Therefore, the unobserved set \(T \setminus S\) becomes empty, i.e., \(T \setminus S \to \emptyset\). Combined with the expressiveness of DeepNT and the completeness of path sampling, the model generalizes well to unobserved pairs \(\langle u, v \rangle \in T \setminus S\). This ensures that the generalization error also diminishes:
\[
\mathbb{E}_{\langle u, v \rangle \sim T \setminus S}[\lvert \hat{y}_{uv} - y_{uv} \rvert] \to 0.
\]

Combining these results, the total error for all node pairs in \(T\), which is the sum of the training error and the generalization error, converges to zero:
\[
\epsilon_{\text{total}} = \mathbb{E}_{\langle u, v \rangle \sim S}[\lvert \hat{y}_{uv} - y_{uv} \rvert] + \mathbb{E}_{\langle u, v \rangle \sim T \setminus S}[\lvert \hat{y}_{uv} - y_{uv} \rvert] \to 0.
\]

Finally, as \(N \to \infty\) and \(S \to T\), DeepNT-\(\mathcal{SP}\) converges to DeepNT-\(\mathcal{AP}\), and the predicted metrics \(\hat{y}_{uv}\) converge in expectation to the true metrics \(y_{uv}\):
\[
\lim_{S \to T} \lim_{N \to \infty} \mathbb{E}_{\langle u, v \rangle \sim T} \left[\lvert \hat{y}_{uv} - y_{uv} \rvert \right] = 0.
\]
This completes the proof.
\end{proof}

\end{document}